\documentclass[journal,twocolumn]{IEEEtran}

\usepackage{amsmath,amssymb,amsfonts}
\usepackage{amsthm}
\usepackage{mathtools}
\usepackage{graphicx}
\usepackage{subcaption}
\usepackage{cite}
\usepackage{algorithm, algorithmic}
\usepackage{booktabs}
\usepackage{hyperref}
\usepackage{color}
\usepackage{mdframed}
\usepackage{tikz}

\theoremstyle{plain}
\newtheorem{theorem}{Theorem}

\newtheorem{proposition}{Proposition}
\newtheorem{corollary}{Corollary} 

\newtheorem{assumption}{Assumption}

\theoremstyle{definition}

\theoremstyle{remark}

\begin{document}

\title{Adaptive Measurement Allocation for Learning Kernelized SVMs Under Noisy Observations}



\author{
Artur Miroszewski\\
$\Phi$-lab, European Space Agency (ESA/ESRIN), Frascati, Italy\\
\texttt{artur.miroszewski@esa.int}
}

\maketitle

\begin{abstract}
Kernel methods are typically formulated under the assumption of exact, noise-free access to the Gram matrix. However, in emerging settings each kernel entry must be inferred from noisy observations, and its accuracy depends on how a limited measurement budget is allocated. Despite this, existing approaches overwhelmingly rely on uniform allocation, which equalizes estimator variance but ignores the highly non-uniform dependence of kernelized classifiers on the Gram matrix.

In this work, we formulate measurement allocation for noisy kernel estimation as a task-aware optimization problem tailored to kernelized Support Vector Machines (SVMs). We derive a variance-aware allocation framework that combines classifier sensitivity with estimator uncertainty, leading to a Neyman-type allocation rule for measurement-based kernels and a Bernoulli specialization relevant to quantum kernel estimation. Building on this analysis, we develop an adaptive measurement allocation strategy that combines margin sensitivity and active set instability, concentrating measurements on the most classifier-relevant regions of the kernel matrix.

Theoretical analysis reveals distinct allocation regimes governed by the heterogeneity of the induced allocation weights, identifying conditions under which adaptive or uniform strategies are preferable. Experiments on synthetic and quantum-kernel datasets demonstrate improved classifier fidelity relative to uniform allocation, while a dual coefficient stability criterion enables substantial measurement savings through early stopping. Together, these results establish adaptive measurement allocation as an effective alternative to uniform sampling for learning with noisy kernels, improving both predictive accuracy and measurement efficiency.
\end{abstract}


\section{Introduction}
\IEEEPARstart{K}{ernel} methods \cite{hofmann2008kernel} are usually presented in a regime where the kernel matrix $K$ is available exactly, up to deterministic numerical precision.
In classic settings, each kernel entry $K_{ij} = k(x_i, x_j)$ is computed by evaluating a closed‑form function, and its value is treated as noise‑free. 
However, a growing number of modern applications depart fundamentally from this assumption. 
In such settings, the accuracy of $K_{ij}$ is determined by the number of measurements allocated to its estimation.
A prominent example arises in quantum machine learning \cite{havlivcek2019supervised, schuld2019quantum}, where kernel entries are estimated from repeated Bernoulli measurements.
Each evaluation of $K_{ij}$ is an empirical, sequential random process, and reducing the variance requires allocating more measurement budget.
These challenges become even more pronounced in high-dimensional quantum models, where phenomena such as barren plateaus \cite{mccleanBarrenPlateausQuantum2018} and kernel concentration \cite{thanasilpExponentialConcentrationQuantum2024} can significantly increase the resources required for reliable training \cite{miroszewski2024searchquantumadvantageestimating}.

\begin{figure}[t!]
    \centering
    \includegraphics[width=0.9\linewidth]{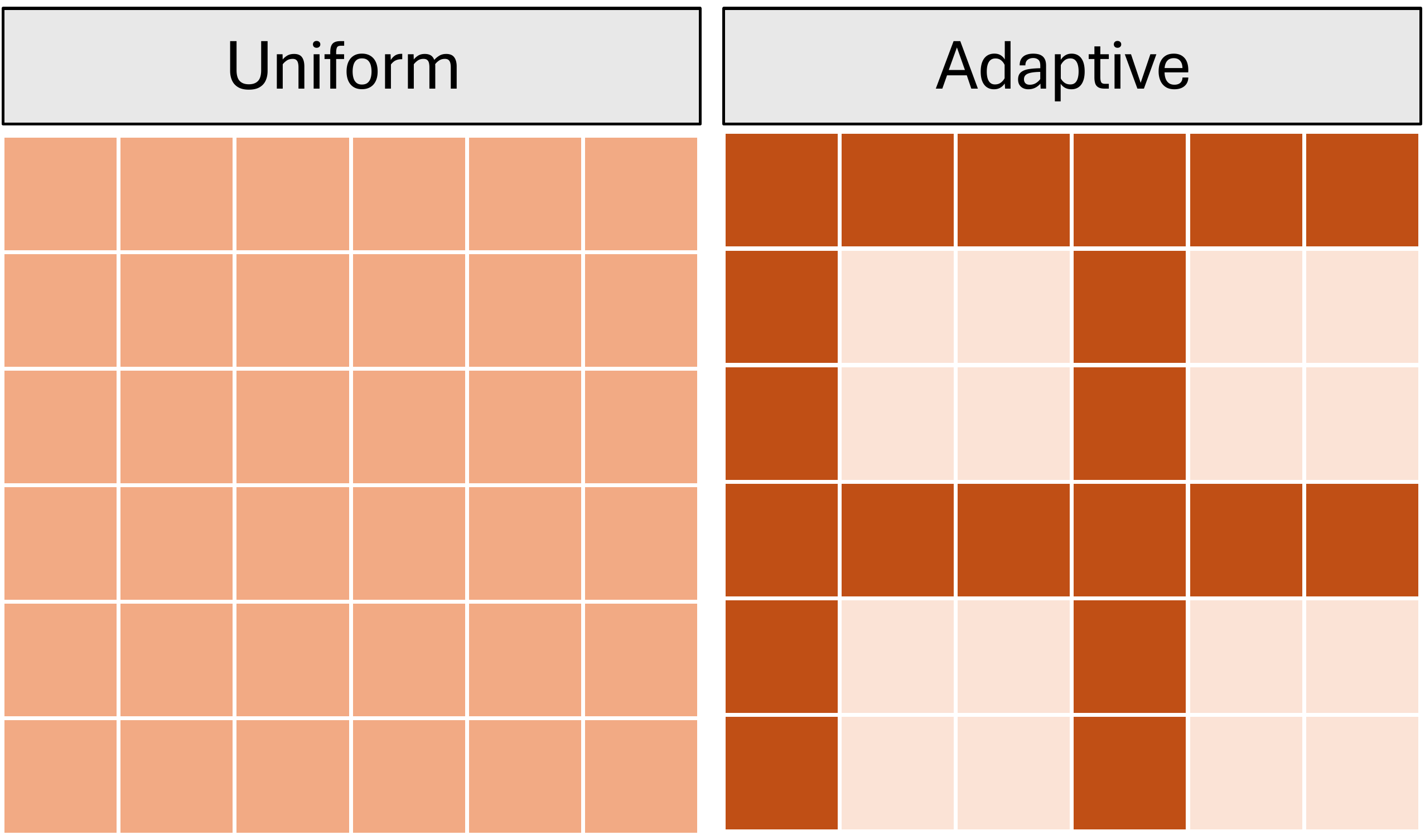}
    \caption{Conceptual illustration of measurement allocation strategies. The uniform scheme distributes measurements evenly across all kernel entries, while our adaptive approach concentrates them in the most influential regions of the kernel, typically involving support vectors and near‑margin points, yielding a classifier that more closely matches the true‑kernel SVM under the same or smaller measurement budget.}
    \label{fig:general_idea}
\end{figure}

Yet, measurement resources are inevitably limited. 
This creates an unavoidable trade‑off: with a fixed global budget, one must decide how to allocate measurements across the independent kernel entries in order to obtain the most accurate downstream machine learning model.
Although quantum kernels provide a main motivating example, the problem studied in this work is not specific to quantum computing. More generally, kernel values may be unavailable in closed form, expensive to compute exactly, or accessible only through finite stochastic estimation procedures. Classical examples include Monte-Carlo kernel approximations and random-feature methods, where kernels are represented as expectations over randomized feature maps and approximated by empirical averages. Random Fourier Features, for instance, approximate shift-invariant kernels using Monte-Carlo samples from the spectral representation of the kernel, with the approximation quality improving as the number of sampled features increases \cite{rahimi2007random,liu2021random}. Thus, quantum kernel estimation can be viewed as one instance of a broader class of measurement-based kernel-learning problems, in which limited estimation resources must be allocated across kernel entries in a way that is useful for the downstream learning task.

The standard approach in the literature is uniform sampling, where each kernel entry receives the same number of measurements. 
Uniform allocation is conceptually simple and symmetric, with well-understood variance properties: it minimizes the maximum variance among the entries and equalizes the per‑entry mean‑square error. 
It also ensures that the Frobenius‑norm error is evenly spread across the Gram matrix and avoids introducing structural bias. 
However, uniform allocation is task‑agnostic: it treats all entries of the kernel matrix as equally important, despite the fact that downstream learning algorithms rarely depend on the kernel in a uniform way.

For kernelized SVMs, the dependence of the learned classifier on the Gram matrix is highly non-uniform \cite{hofmann2008kernel}. Only a subset of training points, the support vectors (the active set), determine the separating hyperplane, and the influence of a kernel entry $K_{ij}$ on the classifier geometry scales proportionally to $\alpha_i \alpha_j y_i y_j$. This suggests that measurement resources should not be distributed uniformly across the kernel matrix, but instead concentrated on entries that are most relevant to the decision boundary. Such a task-aware allocation strategy is challenging because the support-vector structure is initially unknown, the kernel estimates themselves are noisy, and the importance of each entry depends on the evolving classifier learned from those estimates.

In this work, we introduce an adaptive measurement allocation strategy for learning kernelized SVMs from noisy measurements, grounded in the geometric structure of the SVM solution. The method proceeds in two stages. First, a pilot round collects a modest number of measurements for all kernel entries, providing an initial estimate of the kernel matrix and the corresponding dual coefficients. Subsequently, the algorithm performs a sequence of adaptive rounds, where the remaining measurement budget is reallocated according to a score that combines: (i) the sensitivity of the SVM margin to perturbations of individual kernel entries, and (ii) an explicit estimate of active set (the support vectors) instability, capturing the likelihood that points may enter or leave the support-vector set. Together, these signals define a task-aware allocation strategy that focuses resources on entries most relevant to the learned classifier.

The outcome is a measurement strategy that preferentially refines the kernel entries where accuracy matters most, improving support‑vector recovery, margin estimation, and decision‑function stability under fixed measurement budgets. Empirically, we observe consistent improvements over uniform allocation in both kernel reconstruction error and classification performance across a range of datasets and noise regimes.

An open-source research software framework for adaptive measurement allocation is provided through the \texttt{ShotWise} software package to simplify reproducible research and practical deployment.

In the quantum kernel literature, individual measurements are commonly referred to as \textit{shots}. Throughout this paper, we use the more general term, \textit{measurement}, unless discussing quantum implementations specifically. 

\subsection{Contributions}
Our contributions are as follows:
\begin{itemize}
    \item We formulate measurement allocation for noisy kernel estimation as a task-aware problem tailored to kernelized SVMs, highlighting the non-uniform dependence of the classifier on kernel entries.
    \item We establish consistency results showing that the adaptive allocation converges to the corresponding oracle Neyman allocation under mild regularity conditions, yielding asymptotic oracle efficiency.
    \item We introduce a novel allocation criterion combining margin sensitivity and active set instability, providing an interpretable alternative to gradient-based allocation strategies.
    \item We develop a multi-round adaptive algorithm that progressively refines kernel estimates in the most influential regions of the Gram matrix.
    \item We introduce a dual coefficient stability criterion for early stopping, enabling substantial measurement savings while maintaining near-optimal classifier fidelity.
    \item We empirically demonstrate that the proposed approach improves support-vector recovery, margin and decision function estimation of the classifier under fixed measurement budgets.
    \item We release \texttt{ShotWise}, an open-source Python package implementing the proposed adaptive measurement-allocation framework, including support for simulation, quantum hardware execution through \texttt{Qiskit} \cite{javadi2024quantum}, and reproducible benchmarking workflows.
\end{itemize}
\subsection{Related Work}
A related line of work studies scalable kernel approximation through randomized feature maps. Random-feature methods replace implicit kernel evaluations by finite-dimensional randomized embeddings whose inner products approximate the target kernel \cite{rahimi2007random}. In particular, Random Fourier Features approximate shift-invariant kernels by Monte Carlo sampling from the kernel's spectral representation, and subsequent work has analyzed approximation error, variance reduction, and the number of random features needed to preserve downstream learning performance \cite{liu2021random}. More recent work has also considered quasi-Monte Carlo feature constructions to improve the convergence rate of kernel approximation \cite{huang2024quasi}. These approaches primarily focus on global kernel approximation or scalable learning through a shared random feature map. By contrast, the present work studies task-aware allocation of finite measurement resources across individual kernel entries, guided by the geometry and stability of the downstream SVM solution.


Low-rank approximation techniques, such as Nystr\"om-based methods, have been adapted to quantum kernels to reduce the number of required circuit evaluations while preserving downstream performance \cite{coelho2025quantum}.
Similarly, kernel matrix completion approaches exploit structural assumptions, such as approximate low-rankness, to infer missing entries and reduce the need for direct quantum measurements \cite{naveh2021kernel}. These methods focus primarily on reducing the number of kernel evaluations by exploiting global structure in the Gram matrix.

Shot-efficient quantum kernel learning has also been studied from the perspective of reliability of the downstream classifier. In particular, recent work analyzes how finite-shot noise in kernel entries propagates to the SVM solution and derives bounds on the number of measurements required to preserve the margin and classification performance \cite{shastry2022shot}. These approaches focus on establishing global shot-complexity guarantees under uniform estimation strategies, rather than optimizing the allocation of measurements across individual kernel entries.

Adaptive measurement-allocation strategies have also been studied extensively in the context of variational quantum algorithms (VQAs), where finite measurement resources must be distributed across estimates of gradient components during training. Notable examples include the iCANS \cite{Kubler2020adaptiveoptimizer} and gCANS \cite{gu2021adaptive} frameworks, as well as the Rosalin optimizer \cite{arrasmith2020operator}, which allocate shots dynamically according to estimated gradient magnitudes and variances in order to improve convergence while reducing overall measurement cost. Although these methods address a different learning setting than kernel methods, they share the common principle of task-aware measurement allocation, namely that measurement resources should be concentrated where they produce the greatest improvement in downstream learning performance. Our work extends this philosophy to kernelized SVMs, where the fundamental objects are kernel entries rather than optimization gradients, and where allocation decisions are guided by classifier geometry and support-vector stability instead of gradient-estimation uncertainty.

More recently, these ideas have been extended to quantum kernel methods through adaptive allocation schemes operating directly at the level of kernel entries. In particular, a very recent work introduces Active Quantum Kernel Acquisition (AQKA) \cite{xu2026aqka}, which derives allocation rules from the sensitivity of a global training objective with respect to individual kernel entries. The method utilizes gradient-based and KKT-inspired conditions to compute closed-form allocations under a fixed measurement budget, leading to a framework grounded in optimal experimental design.

While both AQKA and our approach depart from uniform sampling and exploit the non-uniform relevance of kernel entries, the underlying perspectives differ. AQKA formulates measurement allocation as a global optimization problem driven by gradients of a surrogate loss function, whereas our approach focuses specifically on the geometric structure of kernelized SVMs. In particular, we derive allocation signals directly from the dual coefficients and the stability of the support-vector set, explicitly targeting margin sensitivity and active set uncertainty. This leads to a simpler and more interpretable allocation mechanism that is tightly coupled to the SVM decision boundary, rather than to a generic objective function.

Orthogonally to kernel estimation, recent work has investigated the computational complexity of inference in quantum kernel methods. In particular, \cite{gil2026optimal} shows that by encoding the full decision function as a single observable and employing quantum amplitude estimation, the query complexity of inference can be reduced from $\mathcal{O}(N)$ to $\mathcal{O}(|\alpha|_1/\epsilon)$, which is provably optimal. These results address the cost of evaluating trained models, and are complementary to our focus on shot-efficient estimation of the kernel matrix itself.

\subsection{Manuscript structure}
The remainder of the paper is organized as follows. Section~\ref{sec:background} introduces the measurement-based kernel framework and reviews the SVM sensitivity and active set analyses that support our approach. Section~\ref{sec:problem_formulation} develops the theoretical measurement-allocation framework, including a general variance-aware allocation rule, a Bernoulli specialization, and a comparison with uniform allocation. Section~\ref{sec:adaptive} presents the adaptive allocation algorithm and early-stopping criterion. Section~\ref{sec:results} reports empirical results on synthetic and quantum-kernel datasets. Section~\ref{sec:shotwise} describes the \texttt{ShotWise} software package and reproducibility resources. Finally, Section~\ref{sec:conclusion} concludes the paper.

\section{Background}\label{sec:background}
In this section, we introduce the theoretical ingredients underlying the proposed measurement-allocation framework. 
We first review Support Vector Machines (SVMs), with particular emphasis on the geometric quantities that determine the influence of kernel entries on the learned classifier. 
We then introduce the notion of measurement-based kernels, in which kernel entries are observed through noisy estimation procedures, and conclude with a Bernoulli specialization that captures the finite-measurement kernel estimation setting encountered in many applications.
\subsection{Support Vector Machines}

\noindent Support Vector Machines (SVMs) \cite{cortes1995support} are large‑margin classifiers that construct a decision function by focusing on the most informative training points, known as \emph{support vectors}. These support vectors lie closest to the decision boundary and fully determine both the geometry of the classifier and its predictive behavior. One of the central strengths of SVMs is that their optimization problem can be expressed entirely in terms of pairwise inner products between data points. This enables the use of \emph{kernels} to implicitly map the data into a high-dimensional feature space without explicitly computing feature embeddings, a mechanism widely known as the \emph{kernel trick}.

Given training samples $\{(x_i, y_i)\}_{i=1}^n$ with labels $y_i \in \{\pm 1\}$ and a positive semidefinite kernel function $K(\cdot,\cdot)$, the SVM is trained by solving the dual optimization problem
\begin{equation}
\label{eq:svm_dual}
\begin{aligned}
\max_{\alpha}\quad 
    & \sum_{i=1}^n \alpha_i 
    \;-\;
    \frac{1}{2}\sum_{i=1}^n\sum_{j=1}^n 
        \alpha_i \alpha_j y_i y_j K(x_i, x_j),
\\[1mm]
\text{s.t.}\quad
    & 0 \le \alpha_i \le C,\qquad
      \sum_{i=1}^n \alpha_i y_i = 0,
\end{aligned}
\end{equation}
where the variables $\alpha_i$ are the \emph{dual coefficients}. These coefficients measure the contribution of each training sample to the decision boundary. At optimality, only a subset of coefficients satisfies $\alpha_i > 0$; the associated samples constitute the support vectors. Because the dual formulation depends exclusively on $K(x_i, x_j)$, the SVM directly incorporates similarities between samples without requiring explicit feature representations.

Once the optimal coefficients $\alpha$ are obtained, the decision function takes the form
\begin{equation}\label{eq:decision_function_optimal}
f(x) 
    = \sum_{i=1}^n \alpha_i y_i K(x_i, x) + b,
\end{equation}
where $b$ is the bias recovered from the Karush--Kuhn--Tucker conditions. The signed distance to the decision boundary is given by $f(x)$, and classification is determined by $\operatorname{sign}(f(x))$.

The margin of the classifier $1/\|w\|$ quantifies the separation between the two classes. The norm can be expressed in terms of the dual variables as
\begin{equation}\label{eq:margin_optimal}
\|w\|
    = \sqrt{
     \sum_{i,j}
        (\alpha_i y_i )
        K_{ij}
        (\alpha_j y_j)
      }.
\end{equation}



\subsection{Measurement‐Based Kernels}

Let $K \in [0,1]^{n \times n}$ denote the true, symmetric, positive semi-definite kernel matrix. In practice, the learner does not have direct access to the entries $K_{ij}$, but instead accesses kernel entries through a stochastic measurement procedure, in which noisy realizations of the kernel are prepared and sampled.

The effective kernel values are intended to approximate the true kernel, but may be impacted by hardware noise, as is typical in near-term quantum computing devices \cite{preskillQuantumComputingNISQ2018, vukvsic2026comparative}. We model this by introducing a stochastic effective kernel process $K^{(k)}$, where each realization corresponds to one execution of the underlying kernel-entry evaluation procedure:
\begin{equation}\label{eq:hardware_noisy_k}
    K^{(k)}_{ij} = K_{ij} + \varepsilon^{(k)}_{ij},
\end{equation}
where the noise term $\varepsilon^{(k)}_{ij}$ represents hardware-induced fluctuations satisfying
\begin{equation}
    \mathbb{E}[\varepsilon^{(k)}_{ij}] = 0.
\end{equation}
We do not assume that the noise realizations $\varepsilon^{(k)}_{ij}$ are independent across measurements. In particular, correlations between $\varepsilon^{(k)}_{ij}$ and $\varepsilon^{(s)}_{ij}$ for $k \neq s$ are allowed, capturing realistic effects such as calibration drift, coherent errors, or slowly varying hardware fluctuations. 
Define
\begin{equation}
    \sigma^2_{\mathrm{phys},ij} := \mathrm{Cov}( K^{(k)}_{ij}, K^{(s)}_{ij}), \qquad k \neq s,
\end{equation}
where we assume that the kernel noise process admits finite second moments and stationary second-order statistics, i.e. $\mathrm{Cov}( K^{(k)}_{ij}, K^{(s)}_{ij})$ depends only on whether $k=s$ or $k \neq s$.
As we will see, these correlations induce variability in the effective kernel across repeated evaluations that cannot be reduced by averaging alone.

The learner does not observe the effective kernel directly. Instead, each kernel entry is estimated from a finite set of observations, producing an estimator $\widehat K_{ij}$. The resulting estimator variance reflects both the variability of the effective kernel process and the finite precision of the observation procedure. We denote this variance by $\operatorname{Var}(\widehat K_{ij})$. The specific form of the estimator and its variance depend on the underlying measurement model and will be introduced later in Section~\ref{sec:bernoulli}.

\subsection{Sensitivity of the SVM to Kernel Perturbations}\label{sec:theory_sensitivity}

In the measurement-based setting, kernel entries are observed only through noisy estimates $\widehat{K}_{ij}$. Because each kernel entry enters the dual objective explicitly, even small perturbations of particular kernel values can alter the learned classifier.

We assume here that a small perturbation of the kernel matrix, $\delta K$, does not change the identity of the active set. More formally,
\begin{assumption}[Local Active Set Stability] \label{ass:active} 
Let \[ \mathcal A(K) = \Big( \{i:\alpha_i=0\}, \{i:0<\alpha_i<C\}, \{i:\alpha_i=C\} \Big) \] denote the partition of training samples induced by the optimal dual solution of the SVM trained with kernel matrix $K$. There exists a neighborhood $\mathcal N(K)$ of the true kernel matrix such that \[ \mathcal A(K+\delta K) = \mathcal A(K) \] for every $K+\delta K\in\mathcal N(K)$.
\end{assumption}  
Under this condition, the envelope theorem (See App.~\ref{app:envelope}) implies that the derivative of the margin with respect to a kernel entry is the partial derivative of the dual objective evaluated at the optimum, without differentiating through $\alpha$,
\begin{equation}\label{eq:margin_derivative}
    \frac{\partial \|w\|^2}{\partial K_{ij}}
    = \alpha_i \alpha_j y_i y_j.
\end{equation}
Thus the influence of $K_{ij}$ on the classifier geometry is determined by:
\begin{itemize}
    \item whether both $i$ and $j$ are support vectors,
    \item the magnitude of their dual coefficients,
    \item and the label interaction term $y_i y_j$.
\end{itemize}
Kernel entries involving non-support vectors ($\alpha_i = 0$ or $\alpha_j = 0$) have zero first-order influence on the margin. This reveals a strong sparsity structure: only a small subset of kernel entries, those associated with active constraints, contribute meaningfully to the geometry of the classifier. 

\subsection{Active‑Set Instability and Near‑Margin Uncertainty}\label{sec:instability_measure}
Support Vector Machines depend not only on the numerical values of the dual coefficients, but also on the identity of the indices that satisfy the optimality conditions. This set of indices constitutes the active set. In the process of kernel estimation, changes in the kernel matrix can alter the margin geometry and the decision values, potentially causing a point to enter or leave the support‑vector set. Because measurement noise affects $\widehat{K}$ unevenly across entries, the risk of these active set transitions is non‑uniform across training points.
Additionally, unlike smooth variations in the dual variables, changes in this active set constitute discrete transitions that can significantly alter the classifier.

In the context of kernel estimation, perturbations in the kernel matrix modify both the margin geometry and the decision values, potentially causing a point to enter or leave the support-vector set.
At the optimum, each point satisfies one of the complementary slackness conditions:
\begin{equation}
\begin{cases}
\alpha_i = 0, & y_i f(x_i) > 1, \\
0 < \alpha_i < C, & y_i f(x_i) = 1, \\
\alpha_i = C, & y_i f(x_i) < 1.
\end{cases}
\end{equation}
Thus, changes in support-vector membership correspond to threshold-crossing events in the margin residual
\begin{equation}
    \Delta_i = y_i f(x_i) - 1.
\end{equation}
If
\begin{itemize}
    \item $\Delta_i > 0$, the point lies outside the margin and has $\alpha_i = 0$,
    \item $\Delta_i = 0$, the point lies exactly on the margin, with $0 < \alpha_i < C$,
    \item $\Delta_i < 0$, the soft‑margin constraint is active, yielding $\alpha_i = C.$
\end{itemize}

Noise (hardware or measurement) perturbs the decision values $f(x_i)$, so the sign of $\Delta_i$ becomes uncertain whenever the decision function is uncertain. 
Points with small $|\Delta_i|$ lie close to the decision boundary and are therefore most susceptible to changes in support-vector status under perturbations.

Since the decision function, Eq.~\eqref{eq:decision_function_optimal},
is linear in the kernel entries, its variance under measurement noise is
\begin{equation}\label{eq:decision_function_variance}
\mathrm{Var}(\widehat{f}(x_i))
= \sum_{j=1}^n (\hat{\alpha}_j y_j)^2 \,\mathrm{Var}(\widehat K_{ij})    
\end{equation}
This expression is independent of the specific measurement model and depends only on the variance structure of the kernel estimator.

The probability that $x_i$ lies on the margin (or inside it) is
\begin{equation}\label{eq:prob_sv}
    P_i
    := \mathbb{P}(y_i \widehat{f}(x_i) \le 1).
\end{equation}
The quantity $P_i$ can be interpreted as the probability that the point $x_i$ lies on or within the margin under measurement noise, and therefore as a proxy for the likelihood that $x_i$ belongs to, or may transition into, the support-vector set.

This perspective reveals a second, fundamentally different source of sensitivity beyond the local derivatives discussed earlier. While gradient-based sensitivity captures how the objective varies under infinitesimal perturbations, active set instability captures discrete structural changes in the classifier induced by finite perturbations. As a result, points with high transition probability $P_i$ represent regions where additional measurement precision is most valuable, as errors in these regions can alter both the composition of the support-vector set and the resulting decision boundary.

\subsection{Bernoulli Measurement Model}\label{sec:bernoulli}
The preceding developments require only access to the variance of the kernel estimator and are therefore agnostic to the specific observation model. To obtain explicit expressions that can be used for allocation design and empirical evaluation, we now specialize the framework to Bernoulli measurements. Besides being analytically convenient, this model captures the finite-shot estimation process underlying quantum kernel methods, which constitute the primary experimental focus of this paper.

Given a realization of the effective kernel, individual measurements are obtained through Bernoulli sampling:
\begin{equation}
    X^{(k)}_{ij} \sim \mathrm{Bernoulli}(K^{(k)}_{ij}), \qquad k = 1, \dots, N_{ij},
\end{equation}
where $N_{ij}$ is the number of measurements allocated to the pair $(i,j)$. 
The empirical estimator
\begin{equation}\label{eq:K_hat}
    \widehat{K}_{ij} = \frac{1}{N_{ij}} \sum_{k=1}^{N_{ij}} X^{(k)}_{ij}
\end{equation}
is an unbiased estimator of the mean effective kernel,
\begin{equation}
    \mathbb{E}[\widehat{K}_{ij}] = K_{ij}
\end{equation}
which follows from the zero-mean assumption on the noise.

The variance of the estimator reflects two distinct sources of uncertainty: finite-measurement sampling noise and fluctuations of the underlying kernel across repeated evaluations (See App.~\ref{App:variance}),
\begin{equation}\label{eq:K_var}
    \mathrm{Var}(\widehat{K}_{ij}) =
    \frac{K_{ij}(1 - K_{ij})}{N_{ij}}
    +
    \left(1 - \frac{1}{N_{ij}}\right)\sigma^{2}_{\mathrm{phys},ij}.
\end{equation}
The first term corresponds to standard binomial sampling noise and decreases with $N_{ij}$, while the second term represents an irreducible variance component induced by correlations in the underlying kernel process.

The diagonal entries $K_{ii}=1$ are assumed to be known exactly and require no measurement. Moreover, by symmetry $K_{ij}=K_{ji}$, only the independent kernel entries in the upper triangle (excluding the diagonal) must be estimated. We refer to this set as the \emph{independent kernel entries}.

\section{Problem Formulation and Measurement Strategies}\label{sec:problem_formulation}
This section develops the measurement-allocation framework from the perspective of SVM geometry. We derive oracle allocation rules that characterize the optimal distribution of measurements under idealized assumptions and analyze their relationship to the standard uniform allocation. We then bridge these idealized results to practical settings, where the true model is unknown and allocation decisions must be made using noisy estimates, motivating both the adaptive allocation procedure and the instability correction introduced later in the manuscript.

\begin{figure}[t!]
    \centering
    \includegraphics[width=\linewidth]{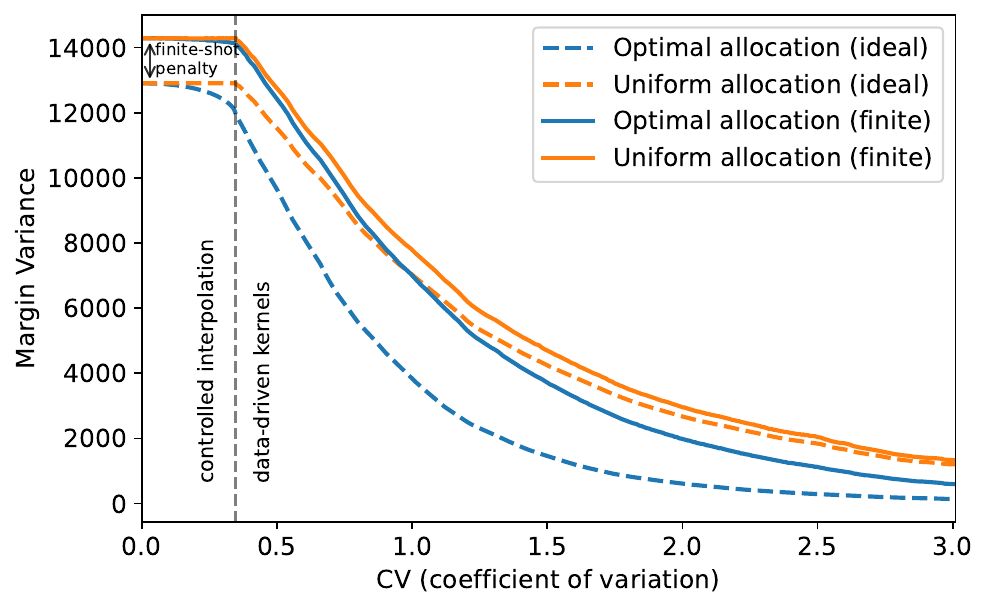}
    \caption{Variance of uniform and optimal measurement allocation schemes as a function of weight heterogeneity.
The top panel shows the margin variance for ideal (dashed) and finite-measurement (solid) implementations of optimal and uniform allocation strategies, plotted against the coefficient of variation $\mathrm{std}(w^B_{ij})/\mathrm{mean}(w^B_{ij})$. The left region (“controlled interpolation”) is constructed to enforce the homogeneous limit $\mathrm{CV} \to 0$, where all weights are equal and both allocation schemes coincide. The right region (“data-driven kernels”) corresponds to weights derived from actual kernel instances.
As heterogeneity increases, the optimal allocation significantly reduces the variance relative to uniform sampling. Finite-measurement effects introduce a systematic upward shift of the optimal curve, most pronounced in the low-CV regime and indicated by the vertical gap between oracle and finite curves (finite-shot penalty), consistent with Eq.~\eqref{eq:finite_penalty}.}
    \label{fig:theory_variances}
\end{figure}
\subsection{Problem Formulation}

\noindent Let $K \in [0,1]^{n\times n}$ denote the true positive‑semidefinite kernel matrix and let
$\alpha,\, f(\cdot),\, \|w\|$ be the corresponding SVM solution obtained from Eq.~\eqref{eq:svm_dual}. This “true‑kernel SVM” serves as the reference model.

In the measurement-based setting, the learner does not have direct access to the exact entries of $K$. Instead, the off-diagonal kernel entries are estimated through a stochastic measurement process, producing estimates $\widehat{K}_{ij}$ whose accuracy depends on the allocated measurement resources.

In practice, measurement resources are limited, and we consider a global constraint on the total number of measurements,
\begin{equation}\label{eq:N_tot}
    \sum_{1 \le i < j \le n} N_{ij} = N_{\mathrm{tot}}.
\end{equation}
Training the SVM on the estimated kernel $\widehat{K}$ leads to a model $\widehat{\alpha},\; \widehat f(\cdot),\; \|\widehat w\|$ which generally deviates from the ideal solution. This discrepancy manifests through multiple mechanisms, including continuous perturbations of the dual coefficients, variations in the classifier margin, fluctuations of the decision function, and discrete transitions in the active set.

These observations motivate a task-aware formulation of the measurement allocation problem.

\begin{mdframed}[linewidth=1pt]
\textbf{Objective.}
Given a fixed measurement budget $N_{\mathrm{tot}}$, choose measurement counts $\{N_{ij}\}$ to maximize the fidelity of the learned SVM classifier relative to the true-kernel SVM, subject to $\sum_{1 \le i < j \le n} N_{ij} = N_{\mathrm{tot}}$.
\end{mdframed}

\subsection{Margin Sensitivity and Optimal Allocation}

The objective stated above is conceptually clear but difficult to handle analytically without specifying a concrete measure of degradation of the classifier. Several candidates can be considered, including perturbations of the dual variables, deviations of the decision function, or changes in the geometric margin. 

In this work, we focus on minimizing the variance of the squared norm of the separating hyperplane, $\|w\|^2$, which is inversely related to the squared margin. It provides a global measure of the classifier’s stability and can be directly related to the kernel matrix through Eq.~\eqref{eq:margin_optimal}. This choice leads to a tractable approximation in which the effect of measurement noise can be explicitly quantified and optimized. 

Importantly, this choice is not unique. Alternative formulations, such as minimizing the variance of the decision function, can be analysed within the same framework. As shown in the following subsections, these different objectives lead to allocation strategies of similar structure, concentrating measurements on pairs of points associated with large dual coefficients. This consistency justifies the use of margin variance as a representative and analytically convenient surrogate.

To make the objective concrete, we aim to choose measurement counts to solve the constrained problem
\begin{align}\label{eq:margin_min}
    \min_{\{N_{ij}\}} \ \mathrm{Var}(\|w\|^2)
    \quad \text{s.t.} \quad \sum_{i<j} N_{ij} = N_{\mathrm{tot}}.
\end{align}

Using the sensitivity result in Eq.~\eqref{eq:margin_derivative}, we approximate the variance of the squared margin via first-order propagation of kernel estimation errors:
\begin{equation}
    \mathrm{Var}(\|\widehat{w}\|^2)
    \approx
    \sum_{i<j} (\alpha_i \alpha_j y_i y_j)^2 \, \mathrm{Var}(\widehat{K}_{ij}).
\end{equation}
Since $(y_i y_j)^2 = 1$, the label dependence disappears in the squared sensitivity term.

The preceding sensitivity analysis identifies the local influence of each kernel entry on the squared norm of the classifier. 
To connect this sensitivity to measurement allocation, we first consider the setting in which the variance of the kernel estimator admits the inverse-budget form 
\begin{equation} \label{eq:noise_model_general}
\operatorname{Var}(\widehat K_{ij}) = \frac{\sigma_{ij}^2}{N_{ij}}, 
\end{equation} 
where $\sigma_{ij}^2$ denotes the intrinsic per-measurement variance associated with entry $(i,j)$. 
This corresponds to the regime in which estimator uncertainty is entirely attributable to finite sampling and therefore decreases proportionally to the number of measurements. 
While this assumption excludes the irreducible variance component appearing in Eq.~\eqref{eq:K_var}, it encompasses a broad class of measurement-based estimators and leads to a transparent Neyman-type allocation rule. 
We therefore adopt it temporarily to isolate the role of measurement allocation. The effects of irreducible fluctuations in the underlying kernel process are revisited in the Bernoulli specialization of Sec.~\ref{sec:Bernoulli_specialization}.

Using the first-order approximation of the squared margin, we obtain 
\begin{equation} 
\operatorname{Var}(\|\widehat w\|^2) \approx \sum_{i<j} (\alpha_i\alpha_j)^2 \operatorname{Var}(\widehat K_{ij}) = 
\sum_{i<j} \frac{(\alpha_i\alpha_j)^2\sigma_{ij}^2}{N_{ij}} . 
\end{equation}

We therefore define the sensitivity allocation weight 
\begin{equation} 
w_{ij} := |\alpha_i\alpha_j|\,\sigma_{ij}. 
\end{equation} 
This quantity combines two factors: the geometric sensitivity of the SVM solution to the kernel entry \(K_{ij}\), and the intrinsic uncertainty with which that entry can be estimated.
The corresponding variance-minimization problem is 
\begin{equation} 
\min_{\{N_{ij}\}} \sum_{i<j} \frac{w_{ij}^2}{N_{ij}} \quad \mathrm{s.t.} \quad \sum_{i<j}N_{ij}=N_{\mathrm{tot}}. 
\end{equation}
Solving this constrained optimization problem yields the 
\begin{equation} \label{eq:optimal_allocation}
N_{ij}^{\star} = N_{\mathrm{tot}} \frac{w_{ij}} {\sum_{p<q}w_{pq}}. 
\end{equation} 
The resulting allocation has the same mathematical form as the classic Neyman allocation from stratified sampling \cite{neyman1992two}. 

We take note, that a natural baseline for estimating the kernel matrix under a finite measurement budget is the \emph{uniform measurement scheme}. In it, the total budget $N_{\mathrm{tot}}$ is distributed evenly across all independent off‑diagonal kernel entries. Each pair $(i,j)$ with $i<j$ receives the same number of measurements,
\begin{equation}
    N_{ij}^{\mathrm{unif}} = \frac{2 N_{\mathrm{tot}}}{n(n-1)}.
\end{equation}
Uniform allocation is the de facto standard in much of the quantum-kernel literature. 
In simulation-based studies, measurement costs are often abstracted away because kernel evaluations can be computed deterministically and at comparatively low cost.
As quantum hardware continues to mature and measurement resources become increasingly important in practical deployments, questions of how to allocate finite budgets efficiently are expected to play a more central role.

For the variance proxy introduced above, Eq.~\eqref{eq:margin_min}, the corresponding variance under uniform allocation is
\begin{equation}\label{eq:V_proxy_uniform}
V_{\mathrm{unif}} = \frac{n(n-1)}{2N_{\mathrm{tot}}} \sum_{i<j} w_{ij}^2,    
\end{equation}
while the optimal allocation yields
\begin{equation}\label{eq:V_proxy_optimal}
V^\star = \frac{\left(\sum_{i<j} w_{ij}\right)^2}{N_{\mathrm{tot}}}.    
\end{equation}

A simple argument based on Cauchy-Schwarz inequality (See App.\ref{app:propositions}) leads to the following proposition.
\begin{proposition}\label{prop:optimal_vs_uniform}
\[
V^\star \le V_{\mathrm{unif}},
\]
with equality if and only if $w_{ij}$ are constant.
\end{proposition}

The gap between optimal and uniform allocation is governed by the dispersion of the weights $w_{ij}$. 
When the weights are heterogeneous, the inequality is strict and the optimal allocation, Eq.~\eqref{eq:adaptive_allocation}, provides a significant reduction in variance. 
Since $w_{ij}$ depends on the dual variables, this regime corresponds to structured SVM solutions in which only a subset of training points contributes strongly to the decision boundary.
To illustrate Proposition~\ref{prop:optimal_vs_uniform}, Fig.~\ref{fig:theory_variances} shows the variance proxies associated with the optimal and uniform allocations for synthetic weight distributions with varying levels of heterogeneity.
Heterogeneity is quantified through the coefficient of variation \( \mathrm{CV} = \frac{\mathrm{std}(w_{ij})} {\mathrm{mean}(w_{ij})}\).
The dashed curves correspond to the theoretical variance proxies given in Eq.~\eqref{eq:V_proxy_uniform} and Eq.~\eqref{eq:V_proxy_optimal}.
As predicted by Proposition~\ref{prop:optimal_vs_uniform}, the two variances coincide in the homogeneous regime (\(\mathrm{CV}\rightarrow 0\)), where all weights are equal. 
As the heterogeneity of the weights increases, the gap between the two curves widens, indicating an increasing advantage of task-aware allocation over uniform sampling. 
This behavior reflects the growing concentration of importance within a subset of kernel entries, making non-uniform measurement allocation progressively more beneficial.

\subsection{Bernoulli Specialization} \label{sec:Bernoulli_specialization}
We now specialize the general allocation framework developed above to the Bernoulli measurement model introduced in Section~\ref{sec:bernoulli}. 
In this setting, the reducible sampling variance associated with kernel entry $(i,j)$ is 
\[ \sigma_{ij}^2 = K_{ij}(1-K_{ij}). \] 
Substituting this expression into the general sensitivity allocation weight yields 
\begin{equation}\label{eq;w_B}
w_{ij}^{\mathrm B} = |\alpha_i\alpha_j| \sqrt{K_{ij}(1-K_{ij})}, 
\end{equation} 
and the corresponding oracle allocation becomes 
\begin{equation} \label{eq:adaptive_allocation}
N_{ij}^{\star} = N_{\mathrm{tot}} \frac{w_{ij}^{\mathrm B}} {\sum_{p<q} w_{pq}^{\mathrm B}}. 
\end{equation}
In the presence of both sampling noise and stochastic hardware-induced fluctuations, the variance of the kernel estimator is given by Eq.~\eqref{eq:K_var}. Substituting this expression yields
\begin{equation} 
\begin{aligned} 
&\mathrm{Var}(\|w\|^2) \approx \\
&\sum_{i<j} (\alpha_i\alpha_j)^2 \left( \frac{K_{ij}(1-K_{ij})}{N_{ij}} + \left( 1-\frac1{N_{ij}} \right) \sigma_{\mathrm{phys},ij}^2 \right). 
\end{aligned} 
\end{equation}

This expression consists of a reducible sampling component and an irreducible hardware component,
\begin{align}
    V_{\mathrm{sampling}} 
    &= \sum_{i<j} (\alpha_i \alpha_j)^2 \frac{K_{ij}(1-K_{ij})}{N_{ij}}, \\
    V_{\mathrm{phys}}
    &= \sum_{i<j} (\alpha_i \alpha_j)^2 \sigma^2_{\mathrm{phys},ij}.
\end{align}

The oracle allocation continues to be determined solely by the reducible sampling component, while the hardware contribution introduces an irreducible variance floor that cannot be eliminated through additional measurements.


Thus, the presence of hardware-induced kernel fluctuations does not alter the structure of the optimal allocation, but introduces a fundamental limit on the achievable accuracy of the estimator. In particular,
\begin{equation}
    \lim_{N_{\mathrm{tot}} \to \infty}
    \mathrm{Var}(\|w\|^2)
    =
    \sum_{i<j} (\alpha_i \alpha_j)^2 \sigma^2_{\mathrm{phys},ij},
\end{equation}
showing that the variance cannot be reduced below a hardware-imposed floor.


\subsection{Decision Function Variance Minimization}

While the previous analysis focused on the variance of the margin, one may alternatively consider minimizing the variance of the decision function. Recall that the decision value at a training point $x_i$ is given by Eq.~\eqref{eq:decision_function_optimal}.
Under the same noise model, the variance of $f(x_i)$ can be approximated as in Eq.~\eqref{eq:decision_function_variance}.
Aggregating the variance over all training points provides a global objective measuring decision-function stability
\begin{equation}
V_{\mathrm{dec}}
=
\sum_{i<j}
(\alpha_i^2 + \alpha_j^2)
\frac{K_{ij}(1-K_{ij})}{N_{ij}},    
\end{equation}
which corresponds to weights
\begin{equation}
\left(w_{ij}^{\mathrm{dec}}\right)^2
\coloneqq
(\alpha_i^2 + \alpha_j^2)\,K_{ij}(1-K_{ij}).    
\end{equation}

The optimal allocation when minimizing the decision function follows from Eq.~\eqref{eq:adaptive_allocation},
\begin{equation}
N_{ij}^{\mathrm{dec}} \propto \sqrt{\alpha_i^2 + \alpha_j^2}\,\sqrt{K_{ij}(1-K_{ij})}.    
\end{equation}

Compared with the sensitivity allocation weight, Eq.~\eqref{eq;w_B}, the above expressions differ quantitatively in its exact dependence on the dual variables. 
However, both formulas exhibit the same qualitative behavior. 
In particular, both allocations assign larger budgets to entries involving large dual coefficients, and therefore concentrate measurements on support vectors. 
Therefore, both objectives lead to allocation strategies that exploit the same underlying geometric structure of the SVM solution. This demonstrates that the specific choice of margin variance as the optimization target is not restrictive, but rather representative of a broader class of objectives that prioritize support-vector regions and near-margin interactions. Consequently, the heterogeneity effects discussed in the previous subsection carry over directly to alternative formulations based on decision-function stability.

\subsection{From Oracle to Adaptive Allocation} 

So far, we have analyzed measurement allocation in an idealized oracle setting. 
In particular, we have assumed that: (i) fractional measurement allocations are admissible, and (ii) the quantities required to construct the allocation, including the parameters of the optimal classifier, are known exactly.

In practice, neither assumption holds. 
Measurement budgets must ultimately be realized through integer-valued allocations, and the quantities entering the allocation rule must be estimated from noisy observations. 
Consequently, the allocation itself is subject to uncertainty and may deviate from the oracle optimum.

The adaptive procedure introduced in the next section addresses this challenge by sequentially refining both the kernel estimates and the corresponding allocation weights through repeated measurement rounds. 
Its objective is to progressively approach the allocation that would be obtained if the underlying classifier were known exactly.

The goal of the present subsection, together with the next one, is to bridge the gap between the idealized oracle analysis and practical measurement allocation. Unless stated otherwise, the results presented below apply to the general noise model of Eq.~\eqref{eq:noise_model_general} and do not rely on the Bernoulli specialization introduced in Sec.~\ref{sec:Bernoulli_specialization}.

We begin by quantifying the impact of allocation errors on the variance objective.

\begin{proposition}\label{prop:alocation_errors}
Let $N_{ij}^\star$ denote the optimal (oracle) allocation given by Eq.~\eqref{eq:adaptive_allocation}, and consider a perturbed allocation
\[
N_{ij} = N_{ij}^\star + \delta N_{ij},
\]
with the constraint $\sum_{i<j} \delta N_{ij} = 0$. Assume that the perturbations $\delta N_{ij}$ are small.

Then, under a second-order Taylor expansion of the variance objective, the expected increase in the variance satisfies
\begin{equation}\label{eq:finite_penalty}
\mathbb{E}[V]
=
V^\star
+
\sum_{i<j} 
\frac{w_{ij}^2}{(N_{ij}^\star)^3}
\, \mathbb{E}\left[(\delta N_{ij})^2\right]
+
\mathcal{O}(\|\delta N\|^3).
\end{equation}
\end{proposition}
The proof is provided in App.~\ref{app:propositions}. Proposition~\ref{prop:alocation_errors} allows us to interpret two effects:
\begin{itemize}
    \item The optimal allocation derived in Eq.~\eqref{eq:optimal_allocation} does not explicitly account for the integer nature of measurement counts. In practice, measurement counts must be integer-valued, which introduces deviations from the ideal allocation, captured by the perturbations $\delta N_{ij}$. Proposition~\ref{prop:alocation_errors} shows that such deviations lead to a systematic increase in the variance due to the convex dependence on $1/N_{ij}$. This effect manifests as an upward shift in the variance proxy of both uniform and optimal allocation strategies, as observed in Fig.~\ref{fig:theory_variances}.
    \item The optimal oracle allocation assumes access to the exact weights $w_{ij}$, while in practice these weights must be inferred from noisy kernel estimates and intermediate models. This introduces an additional source of error not present in uniform allocation, as inaccuracies in the estimated weights lead to misallocation of measurement effort. As a result, in regimes where the underlying structure is weak (i.e., $w_{ij}$ are nearly homogeneous), these estimation errors can dominate the benefits of adaptivity, and uniform allocation may outperform adaptive strategies.
\end{itemize}

Taken together with Proposition~\ref{prop:optimal_vs_uniform}, these results reveal a multiple behavior of measurement allocation strategies. In highly structured settings, where the weights $w_{ij}$ are strongly heterogeneous, the optimal allocation is expected to provide substantial gains over uniform sampling, even in the presence of discretization and estimation noise. In moderately structured regimes, these gains persist but are partially reduced by implementation effects. In contrast, in low-structure regimes where the kernel contributions are nearly uniform, the benefits of adaptivity may diminish and can be outweighed by errors in weight estimation, making uniform allocation preferable due to its robustness. This transition between regimes is explored empirically in Section~\ref{sec:results} and in Fig.~\ref{fig:parameter_map}

Having identified the principal mechanisms that cause practical allocations to deviate from the oracle optimum, we now ask whether these deviations vanish in the large-budget limit. 
The following result provides an affirmative answer. 
Under mild regularity conditions, the adaptive allocation converges to the oracle Neyman allocation as progressively more measurements are collected.


\begin{theorem}[Consistency of Adaptive Allocation] \label{thm:allocation_consistency} 
Assume that: 
\begin{enumerate} 
\item the kernel estimator throughout the adaptive rounds $r$ is consistent, \[ \hat K_{ij}^{(r)} \xrightarrow{p} K_{ij} \qquad \text{for all } i<j; \] 
\item the SVM dual optimum is unique; 
\item Local Active Set Stability (Assumption~\ref{ass:active}) holds. 
\end{enumerate} 
Then the estimated allocation weights satisfy \[ \hat w_{ij}^{(r)} \xrightarrow{p} w_{ij} \qquad \text{for all } i<j. \] 
Consequently, the adaptive allocation asymptotically recovers the oracle Neyman allocation. \end{theorem}

The consistency result concerns the reducible finite-measurement component of the estimator variance. In the presence of an irreducible hardware-induced variance floor, the same argument applies to the limiting effective kernel, while the residual variance cannot be removed by allocation.
Theorem~\ref{thm:allocation_consistency} establishes that the adaptive procedure is not merely heuristic. As progressively more measurements are collected, the kernel estimates converge to the true kernel matrix, which in turn implies convergence of the estimated dual coefficients and allocation weights. The adaptive allocation therefore approaches the oracle allocation that would be obtained with complete knowledge of the underlying kernel.

\begin{corollary}[Asymptotic Oracle Efficiency] \label{cor:efficiency} 
Let \[ V = \sum_{i<j} \frac{w_{ij}^{2}} {N_{ij}} \] denote the variance proxy introduced in Eq.~\eqref{eq:margin_min}. 
Under the assumptions of Theorem~\ref{thm:allocation_consistency}, the variance proxy associated with the adaptive allocation converges to that of the oracle Neyman allocation, i.e., 
\[ \frac{\hat V^{(r)}}{V^\star} \xrightarrow{p} 1. \] 
\end{corollary}

Corollary~\ref{cor:efficiency} shows that the convergence established in Theorem~\ref{thm:allocation_consistency} is meaningful from a statistical perspective. As the measurement budget increases, the variance proxy achieved by the adaptive allocation approaches that of the oracle Neyman allocation. Thus, although practical implementations necessarily rely on estimated quantities, they asymptotically recover the efficiency of the idealized oracle strategy. Proofs of the above results are provided in App.~\ref{app:consistency_limit}.

\subsection{Active Set Instability as Missing Sensitivity}\label{sec:instability}
Finally, we return to a key assumption underlying the oracle analysis, Assumption~\ref{ass:active}: Local Active Set Stability.
As discussed in Sec.~\ref{sec:instability_measure}, perturbations in the estimated kernel matrix may alter the support-vector partition. 
In particular, points located close to the margin and associated with substantial uncertainty in the estimated kernel matrix are natural candidates for support-vector transitions.

When such transitions occur, kernel entries that have negligible sensitivity under the current active set may become influential after a point enters or leaves the support-vector set. 
Consequently, a purely fixed-active-set sensitivity analysis may underestimate the importance of kernel entries involving near-margin points.

This phenomenon is not merely theoretical. 
During the adaptive procedure (introduced in the next section), both the estimated kernel matrix and the corresponding SVM solution evolve across rounds, potentially inducing changes in support-vector membership. Empirical evidence of this behaviour can be observed in Fig.~\ref{fig:metrics-grid}, where support-vector recovery varies across adaptive rounds rather than remaining constant.

\begin{proposition}[Instability as Missing Sensitivity] \label{prop:missing_sensitivity} 
Let \(Z_i\in\{0,1\}\) indicate whether point \(i\) becomes active under a perturbation of the estimated kernel matrix. Assume that 
\[ \Pr(Z_i=1) \approx P_i, \] 
where \(P_i\) is the instability probability defined in Eq.~\eqref{eq:prob_sv}.
Let \(\alpha_i^{new}\) denote the dual coefficient after such a perturbation and define the transition-induced sensitivity 
\[ h_{ij} := \alpha_i^{new} \alpha_j^{new} y_i y_j \, Z_i Z_j . \] 
Under the bounded-dual assumption 
\[ 0 \leq \alpha_i \leq C, \] 
the expected magnitude of this transition-induced sensitivity satisfies 
\begin{equation}
\mathbb E[|h_{ij}|] \lesssim C^2 P_iP_j .    
\end{equation}
\end{proposition}
Proposition~\ref{prop:missing_sensitivity} (proof provided in App.~\ref{app:propositions}) complements the fixed active set sensitivity analysis.
It gives a theoretical motivation for the introduction of an additional term, complementary to the sensitivity-based one, in the allocation strategy proposed in the next section. In this way we cover two sources of uncertainty: smooth perturbations within a locally stable active set region and discrete changes in support-vector membership. This combination allows the adaptive allocation to balance refinement of the current classifier against exploration of regions where the support-vector structure remains uncertain.

\section{Adaptive Measurement Allocation}\label{sec:adaptive}
The theoretical results of the previous section provide guidance on how measurements should be distributed across kernel entries. 
In this section, we translate these insights into a practical adaptive algorithm that alternates between kernel estimation, SVM training, and measurement reallocation. 
The resulting procedure progressively refines the kernel matrix in the most influential regions while incorporating both margin sensitivity and active set instability information.

\subsection{Adaptive Allocation Algorithm Overview}

The central idea of our approach is to treat kernel estimation not as an isolated statistical task, but as an integral component of the downstream SVM optimization. In Section~\ref{sec:problem_formulation}, we derived an optimal measurement allocation that minimizes the variance of the classifier under the assumption of oracle access to the true underlying machine learning model. 

In practice, however, this assumption does not hold: both the kernel and the associated dual solution must be inferred from noisy measurements, and the active set itself may change during the estimation process. The adaptive strategy therefore aims to approximate the oracle allocation in a sequential manner, while explicitly accounting for the uncertainty in the support-vector structure.

The algorithm proceeds as a sequence of measurement--update--retraining cycles:

\begin{enumerate}
    \item \textbf{Pilot estimation.}  
    A small, uniform number of measurements is allocated to all independent kernel entries, producing an initial estimate $\widehat{K}^{(0)}$. This estimate is sufficient to train an initial SVM and extract coarse structural information, in particular approximate support vectors and margin geometry.

    \item \textbf{Sensitivity and instability assessment.}  
    Using the current kernel estimate $\widehat{K}^{(r)}$, the algorithm evaluates which entries of the kernel matrix are most important for the classifier. Two complementary signals are considered:
    \begin{itemize}
        \item \emph{Geometric sensitivity}, derived from the margin derivative, capturing how perturbations of $K_{ij}$ affect the classifier under a fixed active set;
        \item \emph{Active set instability}, capturing the probability that noise in the kernel may induce changes in support-vector membership.
    \end{itemize}

    \item \textbf{Score-based allocation.}  
    These signals are combined into a normalized allocation score for each kernel entry. A portion of the remaining measurement budget is then distributed proportionally to these scores. Since the resulting allocation is generally non-integer, a multinomial sampling step is used to obtain integer measurement counts while preserving the target proportions in expectation.

    \item \textbf{Kernel update and retraining.}  
    The kernel estimate is updated using the newly acquired measurements, and a new SVM is trained to refine the dual coefficients and structural estimates.

    \item \textbf{Early stopping.}  
    The procedure terminates when the measurement budget is exhausted or when the dual coefficients stabilize, indicating convergence of the classifier.
\end{enumerate}

By iteratively refining both the kernel estimate and the allocation strategy, the algorithm progressively concentrates measurements on the most entries with the greatest influence on the classifier. We provide both a high-level schematic of the algorithm in Fig.~\ref{fig:adaptive_schematic} and the algorithm pseudocode in Algorithm \ref{alg:adaptive}.

\begin{figure}[t!]
    \centering
    \includegraphics[width=0.8\linewidth]{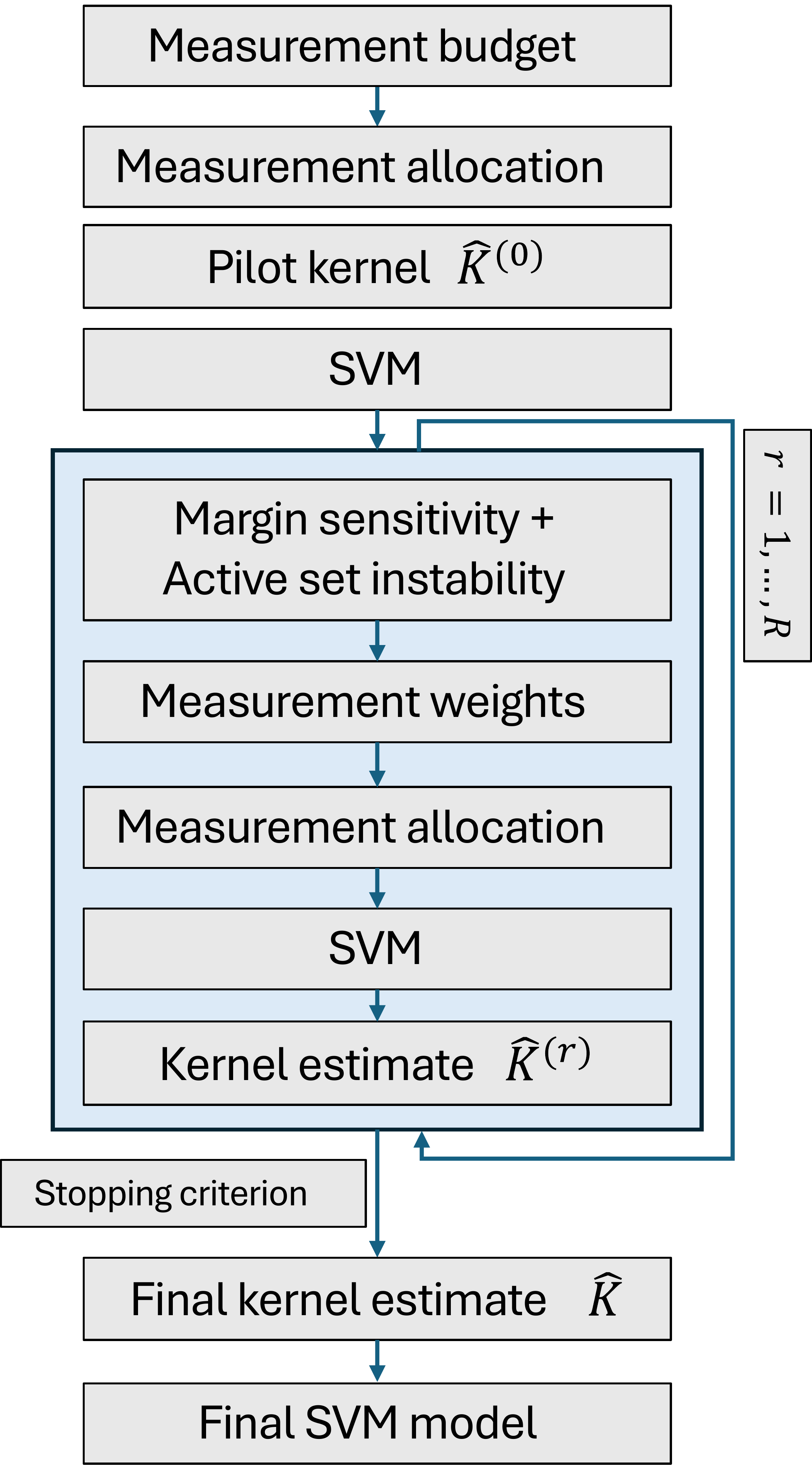}
    \caption{Schematic overview of the adaptive measurement procedure. A small pilot budget produces an initial kernel estimate, which is used to train a preliminary SVM model. Margin‑based sensitivity and active‑set instability are combined to assign measurement weights to individual kernel entries. New measurements are allocated accordingly, the kernel estimate is updated, and the process repeats until the stopping criterion is met, yielding a final SVM model closely matching the true‑kernel solution.}
    \label{fig:adaptive_schematic}
\end{figure}

\begin{algorithm}[t]
\caption{Adaptive Measurement Allocation for Kernelized SVMs}
\label{alg:adaptive}
\begin{algorithmic}[1]
\STATE \textbf{Input:} training data $\{(x_i,y_i)\}_{i=1}^n$, total budget $N_{\mathrm{tot}}$, max rounds $R$, pilot measurements $m_0$, mix weight $\lambda$, tolerance $\varepsilon$
\STATE \textbf{Initialize:} $N_{ij}=0$ for all $i<j$ \quad (measurement counts per entry)

\vspace{0.3em}
\STATE \textbf{Pilot phase:}
\FOR{$i<j$}
    \STATE Allocate $m_0$ measurements to $(i,j)$; set $N_{ij}=m_0$ and $\widehat{K}_{ij} = $ measured fraction of $1$'s
\ENDFOR
\STATE Train initial SVM on $\widehat{K}^{(0)}$ to obtain dual solution $\alpha^{(0)}$

\vspace{0.3em}
\FOR{$r = 1$ to $R$}
    \STATE // using current $\widehat{K}^{(r-1)}$ and duals $\alpha^{(r-1)}$
    \STATE Compute margin residuals $\Delta_i = y_i f(x_i) - 1$ and instability probabilities $P_i$ for all points
    \STATE Compute sensitivity weights $w_{ij}^{\mathrm{sens}}$ for all independent entries $(i,j)$
    \STATE Compute instability weights $w_{ij}^{\mathrm{inst}}$ for all independent entries $(i,j)$
    \STATE Normalize $w^{\mathrm{sens}}$ and $w^{\mathrm{inst}}$ over all $(i,j)$
    \STATE $S_{ij} \gets (1-\lambda)\,\widetilde{w}_{ij}^{\mathrm{sens}} + \lambda\,\widetilde{w}_{ij}^{\mathrm{inst}}$ \hfill \textit{// combined score}
    \STATE Let $B_r$ be the measurement budget for round $r$. Compute $p_{ij} = \frac{s_{ij}}{\sum_{k<l} s_{kl}}$ for all $(i,j)$ and sample $\{\Delta N_{ij}\} \sim \mathrm{Multinomial}(B_r, \{p_{ij}\})$
    \STATE Update $N_{ij} \gets N_{ij} + \Delta N_{ij}$ and $\widehat{K}$ with new measurements; retrain SVM to get $\alpha^{(r)}$
    \STATE $\delta \gets \frac{\|\alpha^{(r)} - \alpha^{(r-1)}\|}{\|\alpha^{(r-1)}\|}$
    \IF{$\delta < \varepsilon$}
        \STATE \textbf{break}
    \ENDIF
\ENDFOR

\STATE \textbf{return} final $\widehat{K}$ and trained SVM
\end{algorithmic}
\end{algorithm}

\subsection{Allocation Scores and Integer Sampling}

The allocation mechanism is based on a pairwise score assigned to each kernel entry. This score combines a theoretically motivated sensitivity term with an instability correction that accounts for active set uncertainty.

\paragraph{Margin sensitivity}
Building upon the sensitivity analysis performed in Sec.~\ref{sec:problem_formulation} we take the optimal allocation score in the Bernoulli measurement-based kernel, Eq.~\eqref{eq:adaptive_allocation} and introduce a sensitivity weight,
\begin{equation}\label{eq:sens_allocation}
w_{ij}^{\mathrm{sens}} \;=\; |\alpha_i \alpha_j| \,\sqrt{\widehat{K}_{ij}(1-\widehat{K}_{ij})},
\end{equation}
which is designed to distribute measurements according to the optimal oracle allocation.

\paragraph{Active set instability}
The sensitivity-based allocation is optimal only when the active set remains unchanged. To account for possible support-vector transitions, we introduce an instability term. 

We take $P_i$ (Eq.~\eqref{eq:prob_sv}), as a proxy for the likelihood that the point $x_i$ belongs to, or may transition into, the support-vector set. In practice, we approximate this value with
\begin{equation}
    P_i \approx \Phi\!\left( -\frac{\widehat{\Delta}_i}{\sigma_{f,i}} \right)
\end{equation}
where $\widehat{\Delta}_i = y_i \widehat f(x_i) - 1$ is the plug‑in margin residual based on the current estimated kernel $\widehat{K}$, $\sigma^2_{f,i} = \mathrm{Var}(\hat{f}(x_i))$ is the variance of the decision function estimation for sample $x_i$ and $\Phi$ is the standard normal CDF.
Based on Propostion~\ref{prop:missing_sensitivity}, we introduce the pairwise instability weight as
\begin{equation}
    w_{ij}^{\mathrm{inst}} \;=\; P_i P_j \, C^2,
\end{equation}
highlighting entries whose noise may induce structural changes in the classifier.

\paragraph{Normalized score}
Since the two components are not inherently comparable, we normalize them before combining. Using, e.g., $\ell_1$ normalization over independent entries,
\[
\widetilde{w}_{ij}^{\mathrm{sens}} = \frac{w_{ij}^{\mathrm{sens}}}{\sum_{k<l} w_{kl}^{\mathrm{sens}}}, 
\quad
\widetilde{w}_{ij}^{\mathrm{inst}} = \frac{w_{ij}^{\mathrm{inst}}}{\sum_{k<l} w_{kl}^{\mathrm{inst}}},
\]
the final allocation score is defined as
\begin{equation}\label{eq:sens_inst}
s_{ij}
=
(1-\lambda)\,\widetilde{w}_{ij}^{\mathrm{sens}}
+
\lambda\,\widetilde{w}_{ij}^{\mathrm{inst}},
\end{equation}
where $\lambda \in [0,1]$ controls the trade-off between variance reduction and structural stability.
An ablation over the mixing parameter \(\lambda\) is reported in App.~\ref{app.ablation_lambda}. The results show that intermediate values of \(\lambda\) often improve both decision-function accuracy and support-vector recovery relative to purely sensitivity-based allocation.

\paragraph{Integer allocation via multinomial sampling}
Given a budget $B$ for the current round, the ideal allocation is proportional to $s_{ij}$, but generally yields non-integer values. To obtain integer measurement counts while preserving the allocation proportions, we define
\[
p_{ij} = \frac{s_{ij}}{\sum_{k<l} s_{kl}},
\]
and draw
\[
\{ \Delta N_{ij} \}_{i<j} \sim \mathrm{Multinomial}(B, \{p_{ij}\}).
\]

This procedure has several important properties:
\begin{itemize}
    \item \emph{Unbiasedness:} $\mathbb{E}[\Delta N_{ij}] = B\, p_{ij}$, matching the desired allocation in expectation;
    \item \emph{Exact budget:} $\sum_{i<j} \Delta N_{ij} = B$ always holds;
    \item \emph{Exploration:} stochastic allocation prevents systematic under-sampling of low-weight entries.
\end{itemize}

After sampling, the new measurements are incorporated into the kernel estimate, and the process repeats.

\subsection{Stopping criterion}
The adaptive procedure iteratively refines the kernel estimate and retrains the SVM after each measurement round. The stopping mechanism is triggered either when the measurement budget is exhausted or when the learned classifier has effectively converged. To detect convergence, we monitor the stability of the dual coefficients. Since both the SVM decision function and the margin depend only on the quantity $y_i \alpha_i$, stabilization of this vector provides a good proxy for convergence of the classifier. In particular, changes in $y_i \alpha_i$ directly translate into changes in both the decision boundary and the margin geometry.

Let $\alpha^{(r)}$ denote the dual coefficients obtained after round $r$, and similarly $\alpha^{(r-1)}$ those of the previous round. We define the relative change between successive iterates as
\begin{equation}\label{eq:stoping_criterion}
    \delta^{(r)}
= 
\frac{
\left\| \alpha^{(r)} - \alpha^{(r-1)} \right\|_2
}{
\left\| \alpha^{(r-1)} \right\|_2 + 10^{-12}
}.
\end{equation}
The small denominator offset prevents numerical issues when the dual coefficients are close to zero in early rounds. When this relative change falls below a prescribed tolerance, $\delta^{(r)} < \varepsilon$, we conclude that further measurements are unlikely to significantly affect the classifier. At this stage, additional refinement of the kernel matrix has diminishing impact on both the decision function and the margin, and the algorithm terminates early. 

If instead the dual coefficients continue to change appreciably, the procedure proceeds to the next round, allocating additional measurements based on the updated sensitivity and instability estimates. This adaptive stopping criterion ensures that measurement effort is not wasted once the classifier has stabilized.

\subsection{Cost Comparison and Practical Regimes}\label{sec:cost_theory}
To compare the efficiency of adaptive and uniform measurement strategies, we model the total computational cost as the sum of a kernel estimation component, proportional to the number of measurement shots, and a standard component, dominated by retraining the SVM.

The kernel estimation cost scales proportionally with the total number of measurements, $c_q N_{\mathrm{tot}}$, while the standard cost scales as $c_c (R+1)n^3$, reflecting the worst‑case complexity of solving the dual SVM problem with a dense kernel matrix \cite{chang2011libsvm}. An adaptive scheme performing one pilot and $R$ refinement rounds therefore requires $R+1$ SVM trainings, while the uniform scheme performs a single training.

The total costs are given by
\[
    C_{\mathrm{uniform}} = c_q N_{\mathrm{tot}} + c_c n^3,
\]
\[
    C_{\mathrm{adaptive}} = c_q r N_{\mathrm{tot}} + (R+1)c_c n^3,
\]
where $r \leq 1$ denotes the fraction of the measurement budget effectively used under early stopping.

To analyze scaling with dataset size, we express the total measurement budget as
\[
N_{\mathrm{tot}} = \frac{n(n-1)}{2}\bar{N},
\]
where $\bar{N}$ is the average number of measurements per kernel entry.

Introducing the relative cost ratio $\tau = c_c / c_q$, the break-even condition $C_{\mathrm{adaptive}} = C_{\mathrm{uniform}}$ yields a critical value
\begin{equation}\label{eq:tau_critical}
    \tau^* = \frac{(n-1)(1-r)}{2n^2 R}\,\bar{N}.
\end{equation}

This expression highlights the fundamental trade-off between measurement cost and standard computation. The kernel estimation component scales as $\mathcal{O}(n^2 \bar{N})$, reflecting the number of kernel entries, while the standard SVM component scales as $\mathcal{O}(n^3)$. As a result, adaptive strategies reduce the estimation cost by concentrating measurements, but incur additional overhead due to repeated model retraining.

The analytical expression above can be interpreted using parameters obtained from the adaptive experiments. In particular, early-stopping results given in Sec.~\ref{sec:early_stop} provide realistic values of the effective shot fraction $r$ and the number of adaptive rounds $R$. Figure~\ref{fig:cost_comparison} visualizes the resulting critical values $\tau^*$ as a function of dataset size $n$ for representative configurations.

The curves separate regimes in which adaptive or uniform measurement is more cost-efficient. For $\tau < \tau^*$, the reduction in measurement cost achieved by adaptive allocation outweighs the additional retraining overhead, making the adaptive strategy favorable. For $\tau > \tau^*$, the classical cost dominates and uniform sampling becomes more efficient.

In typical quantum machine learning settings, the cost of quantum measurements is substantially higher than that of classical computation, implying small values of $\tau$. At the same time, practical applications often involve moderate dataset sizes ($n \approx 10^1$–$10^3$). In this regime, the corresponding critical values $\tau^*$ tend to be significantly larger than realistic values of $\tau$, placing most practical scenarios deep in the adaptive-advantage region.

Consequently, the adaptive measurement scheme offers a dual benefit: it not only improves the accuracy of the learned classifier, as demonstrated in the previous sections, but also reduces overall computational cost by exploiting early stopping and concentrating measurements on the most informative kernel entries.

\begin{figure}[t!]
    \centering
    \includegraphics[width=1.0\linewidth]{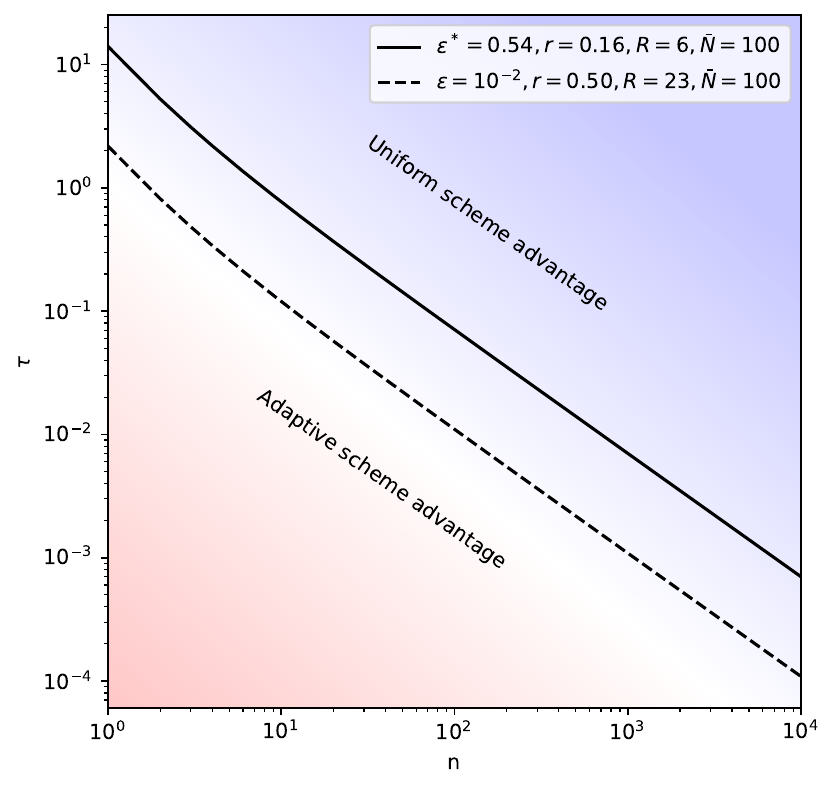}
    \caption{Critical values of the relative cost parameter $\tau$ obtained from the condition $C_{\mathrm{adaptive}}/C_{\mathrm{uniform}} = 1$ as a function of the training data size.
Each curve corresponds to a different choice of stopping threshold $\varepsilon$, effective shot fraction $r$, and number of adaptive rounds $R$.
The curves separate two regimes: below a curve, the adaptive allocation is more cost‑efficient than the uniform strategy; above it, the uniform scheme is cheaper.
Regions close to the curves indicate parameter settings for which the two approaches have comparable total cost.}
    \label{fig:cost_comparison}
\end{figure}

\section{Numerical Results and Discussion}\label{sec:results}
\noindent In this section, we present the empirical evaluation of the proposed adaptive measurement strategy, examining its behavior under fixed‑budget conditions, its performance when equipped with early stopping, and its sensitivity to key hyperparameters. We further relate these findings to the quantum–classical cost model, providing a unified interpretation of adaptive measurement efficiency under practical resource constraints.

\subsection{Experimental Setup and Metrics}

We evaluate the proposed adaptive measurement strategy on both synthetic and application-driven kernel settings. Synthetic datasets are generated from Gaussian mixture models with controllable structure, and kernel matrices are constructed using radial basis function (RBF) kernels. In addition, we consider simulated quantum kernel matrices obtained from real-world data, using feature maps based on block-encoding embeddings applied to the Indian Pines hyperspectral dataset for crop classification~\cite{large_scale}. These kernels correspond to quantum circuits of varying size, ranging from 2 to 100 qubits, providing a realistic spectrum of kernel structures and noise sensitivities. Unless stated otherwise, results are averaged over 1000 independent runs.

Experiments were conducted using \texttt{ShotWise} (\url{https://github.com/ESA-PhiLab/shotwise}), an open-source Python package developed alongside this work to support reproducible research on adaptive measurement allocation for noisy kernel learning. The package implements the proposed algorithms together with utilities for simulation, benchmarking, and quantum hardware execution.

\paragraph{Evaluation Metrics}
We assess performance using several complementary measures:

\emph{Kernel reconstruction.}
We measure the root-mean-square error between the true and estimated kernel,
\[
\mathrm{RMSE}(K) = \sqrt{\frac{1}{n^2} \sum_{i,j} (K_{ij} - \widehat{K}_{ij})^2},
\]
as well as a support-vector-restricted variant $\mathrm{RMSE}(K_{\mathrm{SV}})$ focusing on entries involving true support vectors.

\emph{Support-vector recovery.}
We quantify agreement between true and estimated support-vector sets using the Jaccard index and a weighted variant based on the signed dual coefficients.

\emph{Classifier geometry.}
We measure the relative margin error and the decision-function RMSE, normalized by the true margin. These metrics quantify distortions in the separating hyperplane and deviations in decision values induced by kernel estimation noise.

In settings where a direct comparison between uniform and adaptive strategies is required, we additionally use the relative improvement metric
\begin{equation}\label{eq:relative_RMSE}
\Delta_{\mathrm{RMSE}} =
\frac{\mathrm{RMSE}_f^{\mathrm{uniform}} - \mathrm{RMSE}_f^{\mathrm{adaptive}}}
     {\mathrm{RMSE}_f^{\mathrm{uniform}}},
\end{equation}
which measures the reduction in decision-function error achieved by adaptive allocation relative to the uniform baseline. Positive values indicate that adaptive sampling yields lower error, while negative values correspond to regimes where uniform allocation performs better.

We consider both fixed-budget scenarios, where all methods use the same total number of measurements, and early-stopping settings, where adaptive allocation may terminate before exhausting the budget.

\subsection{Fixed-Budget Performance Experiments}\label{sec:fixed_budget}
\begin{figure*}[t!]
  \centering

  \begin{subfigure}[t]{0.48\textwidth}
    \centering
    \includegraphics[width=\textwidth]{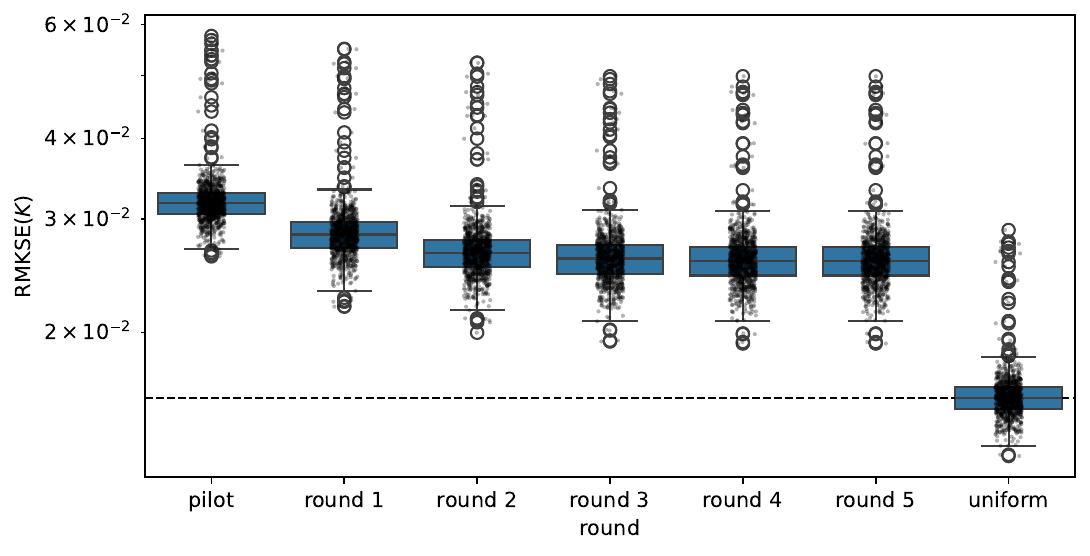}
    \caption{Global kernel error $\mathrm{RMSE}(K)$ (log scale).}
    \label{fig:grid-rmse-global}
  \end{subfigure}\hfill
  \begin{subfigure}[t]{0.48\textwidth}
    \centering
    \includegraphics[width=\textwidth]{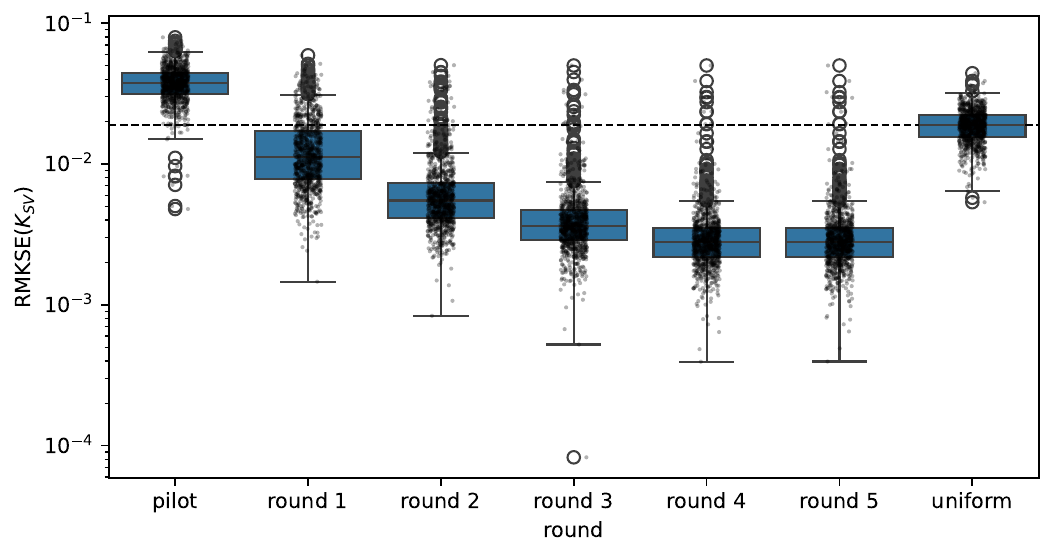}
    \caption{SV-block error $\mathrm{RMSE}(K_{\mathrm{SV}})$ (log scale).}
    \label{fig:grid-rmse-sv}
  \end{subfigure}

  \vspace{0.75em}

  \begin{subfigure}[t]{0.48\textwidth}
    \centering
    \includegraphics[width=\textwidth]{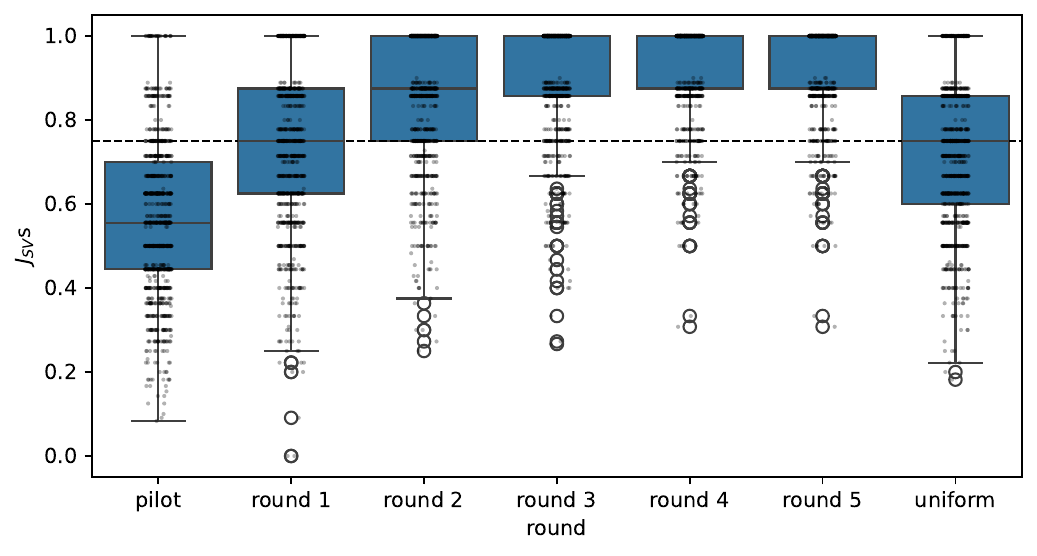}
    \caption{Jaccard similarity of support-vector sets.}
    \label{fig:grid-jaccard}
  \end{subfigure}\hfill
  \begin{subfigure}[t]{0.48\textwidth}
    \centering
    \includegraphics[width=\textwidth]{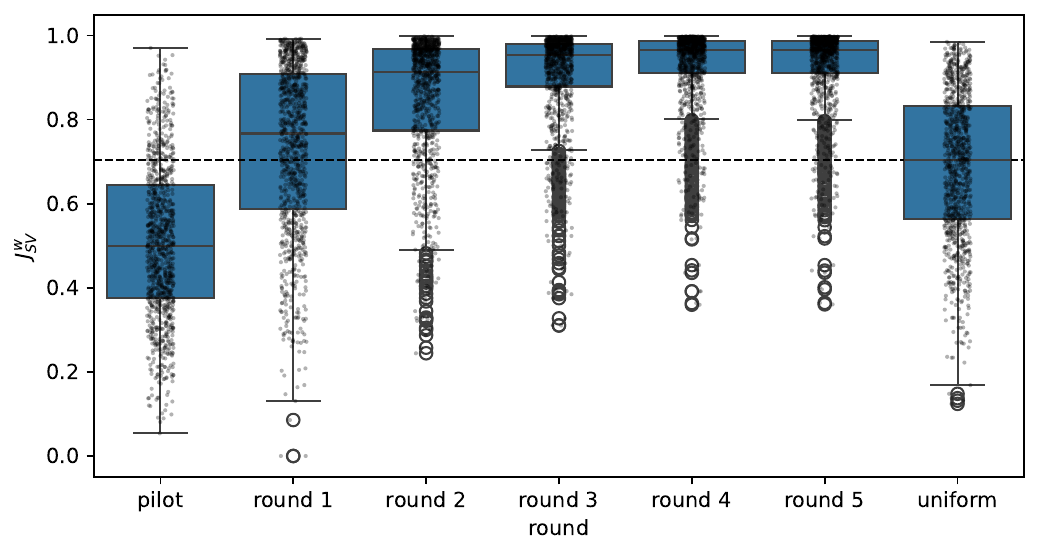}
    \caption{Weighted Jaccard on dual magnitudes $\{|a_i y_i|\}$.}
    \label{fig:grid-weighted-jaccard}
  \end{subfigure}

  \vspace{0.75em}

  \begin{subfigure}[t]{0.48\textwidth}
    \centering
    \includegraphics[width=\textwidth]{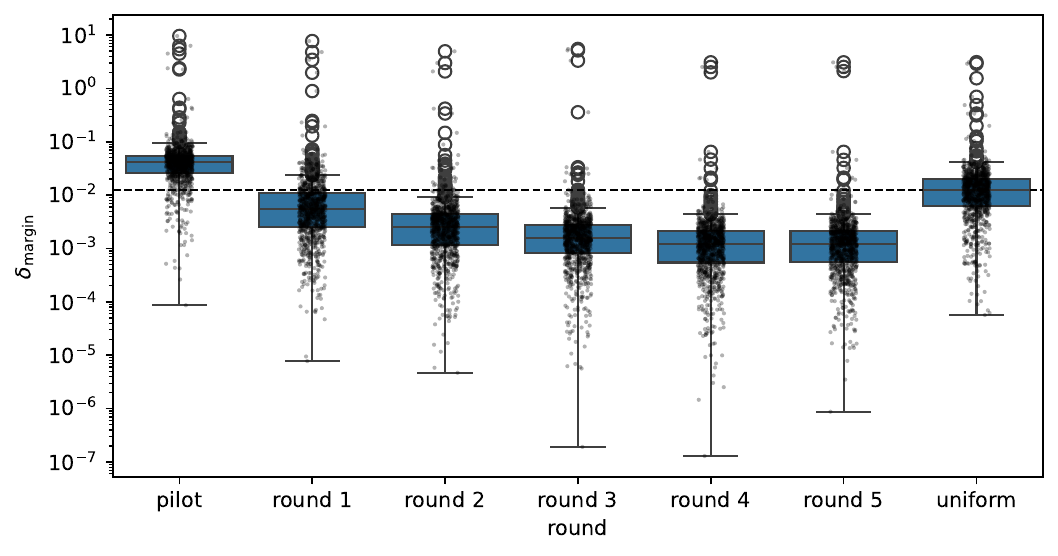}
    \caption{Relative margin error $\delta_{\mathrm{margin}}$.}
    \label{fig:grid-margin-error}
  \end{subfigure}\hfill
  \begin{subfigure}[t]{0.48\textwidth}
    \centering
    \includegraphics[width=\textwidth]{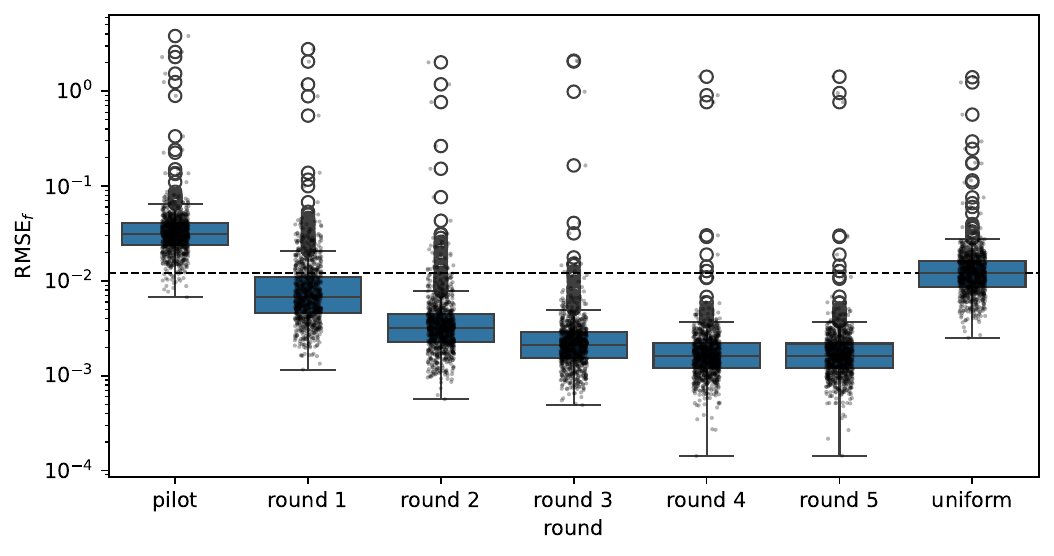}
    \caption{Decision-function $\mathrm{RMSE}$ (normalized by $\|w_{\mathrm{true}}\|$).}
    \label{fig:grid-func-rmse}
  \end{subfigure}

  \caption{
  Summary of metrics comparing \emph{uniform} and \emph{adaptive} measurement schemes
  across the algorithm stages (pilot, round~$1$,\dots, round~$R$).
  \textbf{Row 1:} Global kernel reconstruction error $\mathrm{RMSE}(K)$ and
  SV-block error $\mathrm{RMSE}(K_{\mathrm{SV}})$, both on a logarithmic scale.
  \textbf{Row 2:} Support-vector set agreement (Jaccard, $J_{SV}$) and dual‑weighted
  agreement (Weighted Jaccard, $J_{SV}^w$).
  \textbf{Row 3:} Relative margin error $\delta_{\mathrm{margin}}$ and
  decision-function $\text{RMSE}_f$ normalized by the true margin.
  Horizontal lines show the median of the uniform baseline.
  Support-vector indices correspond to the SVM trained on the true kernel $K$.
  }
  \label{fig:metrics-grid}
\end{figure*}
The fixed-budget experiments reveal a consistent and pronounced advantage of the adaptive measurement scheme over the uniform baseline. Figure~\ref{fig:metrics-grid} summarizes the results across multiple metrics.

A key observation is that uniform allocation achieves lower global kernel reconstruction error, as measured by $\mathrm{RMSE}(K)$, which is expected since it distributes measurements evenly across all entries. However, this comes at the cost of poor accuracy in the most relevant parts of the kernel matrix. In particular, the SV-block RMSE shows that the adaptive strategy reconstructs interactions among true support vectors significantly more accurately. This improvement appears already after the first adaptive refinement round.

This behavior is consistent with the theoretical analysis in Sec.~\ref{sec:problem_formulation}, which predicts that only a subset of kernel entries contributes significantly to the SVM solution. By concentrating measurements on these entries, the adaptive scheme improves task-relevant accuracy rather than global kernel fidelity.

The advantage of adaptive allocation is further reflected in support-vector recovery. Both the Jaccard similarity and weighted Jaccard index demonstrate that the adaptive method more reliably identifies the true active set. The gap is particularly pronounced for the weighted Jaccard metric, indicating that adaptive sampling not only recovers the correct support vectors but also better captures their relative importance in the dual representation.

These improvements propagate directly to classifier-level performance. The adaptive scheme achieves substantially smaller errors in both relative margin and decision-function RMSE, indicating more accurate recovery of the separating hyperplane and decision values. Notably, these gains are immediate: adaptive sampling already outperforms uniform allocation after the first refinement round, and the advantage increases monotonically with subsequent rounds. This rapid improvement reflects the ability of the adaptive strategy to quickly focus measurement effort on the most informative kernel entries.

\subsection{Adaptive convergence dynamics.}
To better understand the behavior of the adaptive process, we examine its evolution over a larger number of rounds. Figure~\ref{fig:saturation_adaptive} shows the decision-function RMSE and dual coefficient stability across up to fifty adaptive rounds.

The decision-function RMSE decreases monotonically, with a characteristic two-phase behavior. In the initial phase, the method exhibits rapid improvement, with a sharp reduction in error during the first several rounds. This is followed by a plateau phase, in which additional measurements yield diminishing returns. This saturation behavior indicates that most of the performance gain is achieved early in the adaptive process.

A closely matching trend is observed in the dual coefficient stability measure. Large changes occur during the initial improvement phase, followed by stabilization as the RMSE curve flattens. Importantly, while RMSE requires access to the true kernel and is therefore not available during execution, the dual-stability metric is fully observable. Its convergence thus provides a practical and theoretically grounded stopping signal.

Taken together, these results demonstrate that adaptive measurement rapidly converges toward the true-kernel solution, and that most of the benefit can be obtained using only a small number of refinement rounds. This observation motivates the early-stopping strategy analyzed in the following subsection.

\begin{figure*}[t]
    \centering
    \includegraphics[width=0.7\linewidth]{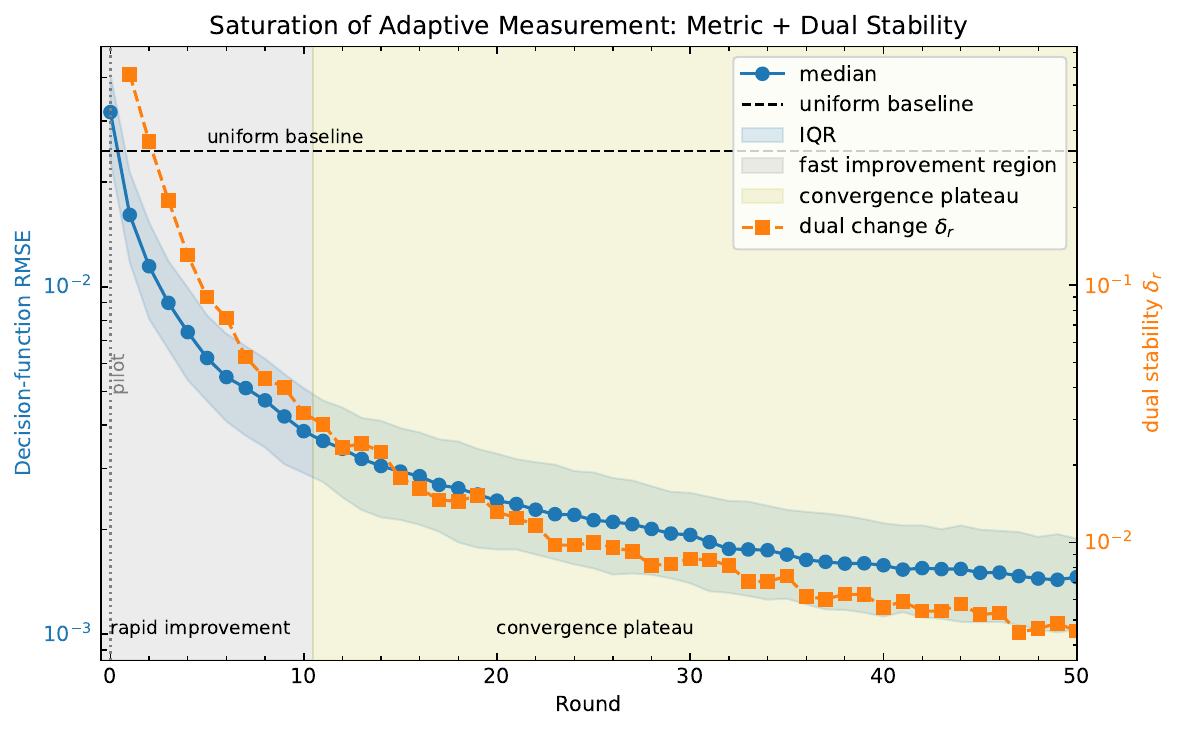}
    \caption{Saturation behavior of the adaptive measurement scheme over $1000$
random \texttt{make\_blobs} datasets of $50$ points. 
The blue curve shows the median decision–function RMSE across adaptive
rounds, with an interquartile ribbon. A dashed horizontal line marks the
performance of the uniform measurement baseline. The shaded region on
the left highlights the phase of rapid improvement (pilot and early
rounds), while the region on the right indicates the convergence
plateau. The orange curve (right $y$–axis, log scale) displays the
median dual–stability measure $\delta_r$, showing that the SVM solution
stabilizes as the adaptive kernel estimates converge.}
    \label{fig:saturation_adaptive}
\end{figure*}

\subsection{Early-Stopping Experiments}\label{sec:early_stop}

The early-stopping experiments evaluate whether the adaptive measurement procedure can terminate before exhausting the full budget while still achieving reliable model quality. The stopping rule is based on the dual coefficient stability introduced in Eq.~\eqref{eq:stoping_criterion}, which measures the relative change of the signed dual vector across consecutive rounds.

To assess the effect of the stopping threshold $\varepsilon$, we sweep a wide range of values and evaluate both accuracy and measurement cost. The primary comparison metric is the relative improvement in decision-function RMSE,
$\Delta_{\mathrm{RMSE}}$, between adaptive and uniform schemes. Positive values indicate an advantage for adaptive allocation, while negative values indicate that uniform sampling performs better.

The results, shown in Fig.~\ref{fig:stopping}, reveal a clear trade-off controlled by $\varepsilon$. For large thresholds, the algorithm terminates early, using only a very small fraction of the measurement budget. In this regime, the adaptive scheme may underperform the uniform baseline due to insufficient refinement of the kernel. As the threshold decreases, the algorithm is allowed to run for more rounds, improving accuracy and eventually surpassing uniform allocation.

A distinct transition occurs at a critical threshold $\varepsilon^* \approx 0.75$, where the sign of $\Delta_{\mathrm{RMSE}}$ changes, indicating a regime in which adaptive allocation consistently outperforms uniform sampling. This transition is confirmed by the success-rate curve, which crosses approximately $50\%$ at the same $\varepsilon^*$ value, marking the boundary between the two regimes.

Beyond classification accuracy, Fig.~\ref{fig:stopping} also quantifies measurement efficiency. At the critical threshold, $\varepsilon^*$, the adaptive scheme uses only about $20\%$ of the full measurement budget and converges in a median of six rounds. This demonstrates a substantial reduction in measurement cost while maintaining competitive or superior performance. For smaller thresholds, the adaptive method achieves further improvements in accuracy while still using significantly fewer measurements than the uniform baseline.

Overall, these results show that dual coefficient stability provides an effective and practically implementable stopping criterion. It enables the adaptive scheme to identify the onset of the convergence plateau observed in Fig.~\ref{fig:saturation_adaptive}, allowing early termination with minimal loss in accuracy. This yields a favorable trade-off between measurement cost and model quality, which is critical in resource-constrained settings.

\begin{figure*}[t]
    \centering
    \includegraphics[width=0.7\linewidth]{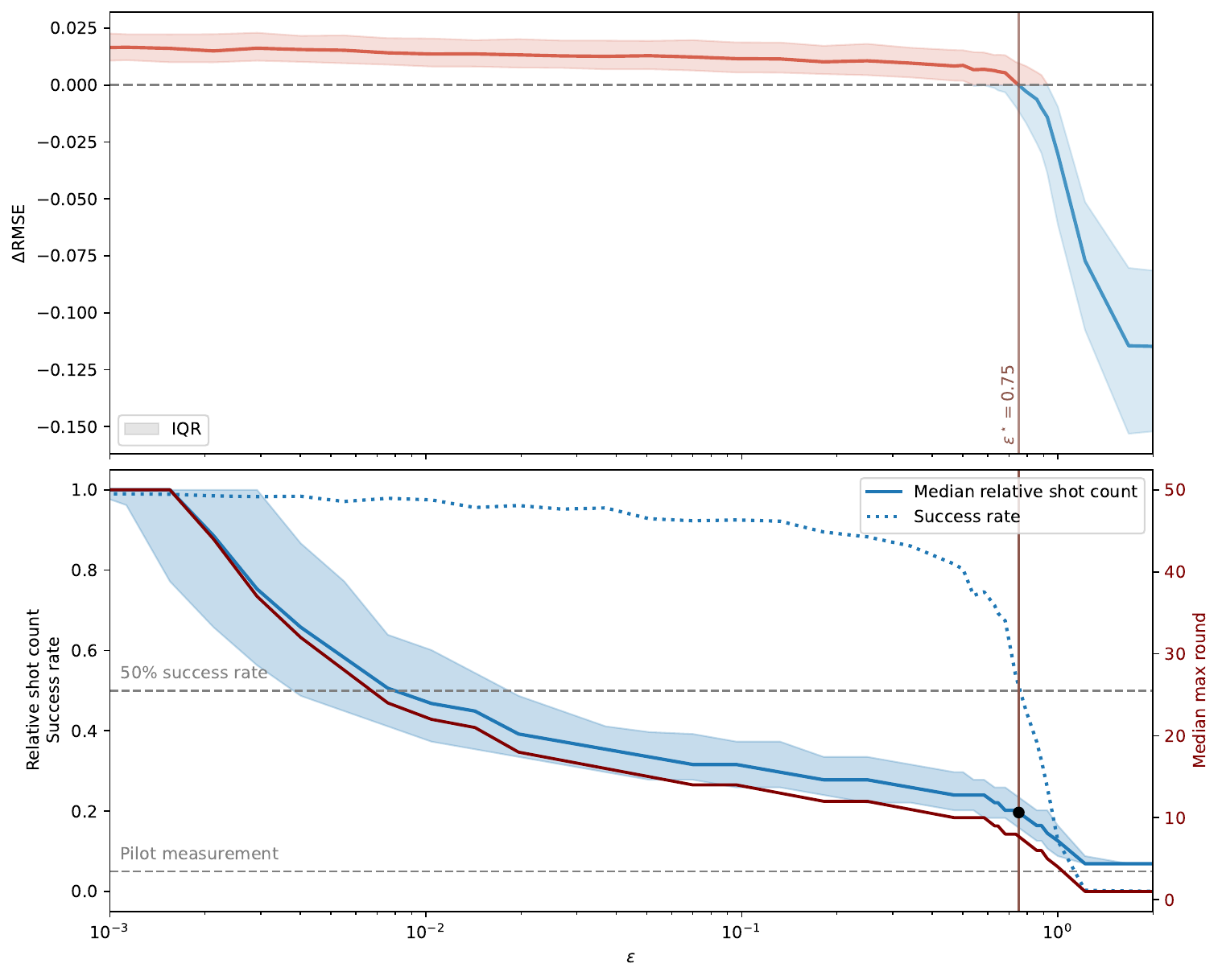}
    \caption{Early stopping via dual–stability.
The adaptive algorithm halts once the dual–stability change $\delta_\ast$ drops 
below a threshold $\epsilon$. \emph{Top:} Difference in decision–function RMSE 
between adaptive and uniform schemes (positive values indicate better adaptive 
performance). \emph{Bottom:} Median relative measurement count, success rate, and median 
number of rounds (right $y$-axis) as functions of $\epsilon$, each computed over 
1000 trials. The vertical line marks the critical $\epsilon$ separating the regimes 
in which adaptive or uniform sampling dominates.}
    \label{fig:stopping}
\end{figure*}

\begin{figure}[t!]
    \centering
    \includegraphics[width=\linewidth]{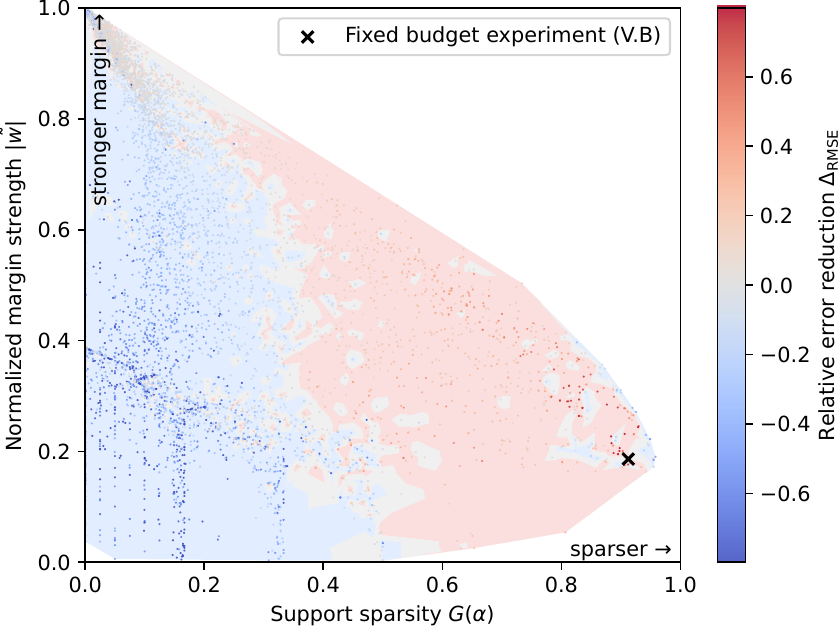}
\caption{
Relative decision-function error reduction $\Delta_{\mathrm{RMSE}}$ for adaptive vs uniform allocation across problem regimes. Each point corresponds to an average over 100 dataset realizations generated with varying structure parameters. Positive values indicate advantage of adaptive allocation, while negative values indicate that uniform sampling perform. The black cross indicates the location of the reference dataset used in previous experiments analysed in Sec.\ref{sec:fixed_budget}.}
    \label{fig:parameter_map}
\end{figure}

\begin{figure*}[t!]
    \centering
    \includegraphics[width=1\linewidth]{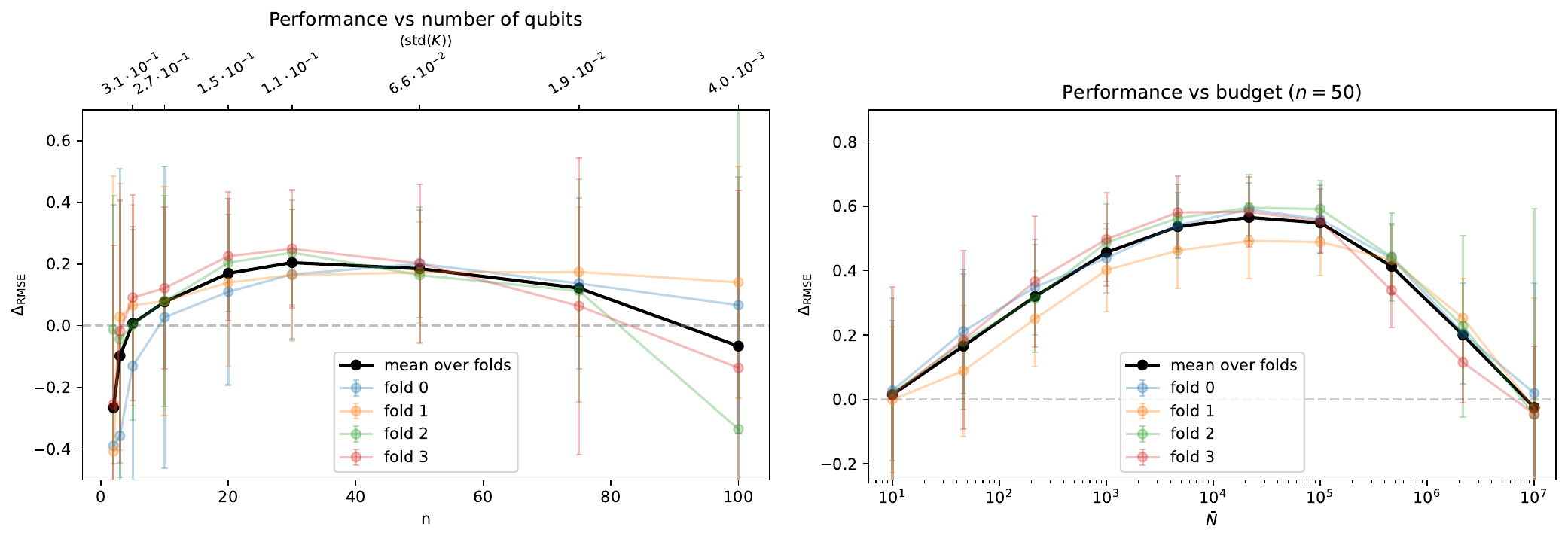}
    \caption{Adaptive vs uniform measurement allocation on quantum kernels derived from the Indian Pines dataset using block-encoding feature maps. 
\textbf{Left:} Relative decision-function error reduction $\Delta_{\mathrm{RMSE}}$ as a function of the number of qubits. The top axis shows the mean standard deviation of the independent kernel entries (averaged over folds), serving as a proxy for kernel concentration issue. Three regimes are visible: (i) a low-structure regime at small qubit numbers where uniform sampling performs slightly better, (ii) an intermediate regime with increased kernel variability where adaptive allocation achieves substantial gains, and (iii) a high-qubit regime where kernel concentration reduces effective signal and both methods exhibit comparable performance.
\textbf{Right:} Relative error reduction as a function of the average number of measurements per kernel entry, $\bar{N}$. For small budgets, both methods are dominated by measurement noise. At intermediate budgets, uniform allocation benefits from its robustness and broad coverage, while at large budgets both approaches converge as the kernel is accurately reconstructed.
Results are averaged over four folds of the dataset. Positive values of $\Delta_{\mathrm{RMSE}}$ indicate an advantage of adaptive allocation, while negative values correspond to superior performance of uniform sampling.}
    \label{fig:quantum_kernel}
\end{figure*}
\subsection{Regimes of dominance: uniform vs adaptive allocation}

The theoretical analysis in Sec.~\ref{sec:problem_formulation} shows that the advantage of the optimal (oracle) allocation is governed by the heterogeneity of the weights $w_{ij}$. In practice, however, these weights must be estimated from noisy kernel measurements, which introduces additional uncertainty. This raises a key question: \emph{in which regimes does adaptive allocation actually outperform uniform sampling?}

To answer this, we perform a systematic exploration of data distributions that control the structure of the induced SVM solution. Synthetic datasets are generated from parametric Gaussian mixtures with tunable properties such as class separation, noise level, anisotropy, and label noise. These parameters directly affect the geometry of the decision boundary and, crucially, the sparsity of the dual coefficients.

To quantify the performance difference between the two strategies, we use the relative decision-function RMSE (defined in Eq.\eqref{eq:relative_RMSE}) where positive values indicate that adaptive allocation improves over uniform sampling.

The results are summarized in Fig.~\ref{fig:parameter_map}.
The horizontal axis measures the sparsity of the SVM solution via the Gini coefficient $G(\alpha)$, which quantifies how unevenly the dual coefficients are distributed: values close to zero correspond to nearly uniform coefficients, while larger values indicate that only a small subset of points (support vectors) carry significant weight. The vertical axis captures the normalized margin strength, reflecting class separation relative to data spread.
Together, these quantities characterize the structural properties determining kernel heterogeneity.

Two distinct regimes emerge:

\begin{itemize}
    \item \textbf{Low-structure regime (small $G(\alpha)$).}  
    In this regime, the dual coefficients are broadly distributed and the weights $w_{ij}$ are nearly uniform. Consequently, the oracle-optimal allocation offers little advantage over uniform sampling. Moreover, because adaptive allocation relies on noisy estimates of these nearly flat weights, estimation errors can dominate, leading to suboptimal measurement allocation. As predicted by Proposition~2, uniform sampling can therefore match or even outperform adaptive strategies due to its robustness.

    \item \textbf{High-structure regime (large $G(\alpha)$ and strong margin).}  
    Here, the SVM solution is sparse and the weights $w_{ij}$ are highly heterogeneous. A small subset of kernel entries carries most of the influence on the classifier. Adaptive allocation exploits this structure by concentrating measurements on these influential entries, resulting in a substantial reduction in decision-function error relative to uniform sampling.
\end{itemize}

These results provide a direct empirical validation of the theoretical predictions: the effectiveness of adaptive allocation is governed by the heterogeneity of the underlying optimization landscape. In highly structured problems, adaptivity yields significant gains, while in near-homogeneous settings its advantage diminishes or disappears due to estimation noise.

\subsection{Adaptive Allocation for Quantum Kernels}

To assess the behavior of the proposed method in realistic quantum machine learning settings, we evaluate adaptive and uniform allocation strategies on quantum kernels derived from the Indian Pines hyperspectral dataset \cite{PURR1947} in the study on the use of large-scale quantum kernels for hyperspectral data classification \cite{large_scale}. The kernels are generated using quantum feature maps, with system sizes ranging from $2$ to $100$ qubits, providing a controlled setting to study the interplay between kernel structure, measurement noise, and resource constraints.

Figure~\ref{fig:quantum_kernel} summarizes the results. On the left, we report the relative decision-function error reduction $\Delta_{\mathrm{RMSE}}$ as a function of the number of qubits, together with the mean standard deviation of kernel entries. On the right, we show the same metric as a function of the average measurement budget per kernel entry, $\bar{N}$.

The left panel reveals a clear three-regime behavior as the number of qubits increases. In the low-qubit regime, uniform allocation performs comparably to, and in some cases slightly better than, the adaptive strategy. Although the empirical variance of individual kernel entries is relatively high in this regime, this variability does not induce meaningful structure in the learning problem. The resulting SVM solutions are diffuse, with broadly distributed dual coefficients and weakly differentiated support vectors, leading to an effectively uniform sensitivity profile across kernel entries. As a result, targeted allocation provides limited advantage over uniform sampling.

In the intermediate regime, corresponding to moderate qubit numbers, a pronounced transition occurs. The kernel begins to induce a more structured decision geometry, and the SVM solution becomes increasingly localized, with a clear separation of influential support vectors and a heterogeneous distribution of effective weights $w_{ij}$. In this setting, adaptive allocation significantly outperforms uniform sampling by concentrating measurements on the most informative kernel entries, leading to substantially improved decision-function accuracy. This regime aligns closely with the conditions identified in Sec.~\ref{sec:problem_formulation}, where heterogeneity in the induced importance weights enables strong gains from structured allocation.

For larger system sizes, the kernel enters a regime dominated by concentration effects \cite{thanasilpExponentialConcentrationQuantum2024}, in which off-diagonal entries shrink and the effective signal available for classification diminishes. In this setting, both methods experience a degradation in performance under a fixed measurement budget, and their performance becomes comparable. However, adaptive allocation remains valuable in practice, as it enables more efficient use of limited resources and allows exploration of deeper concentration regimes before performance collapses. While it does not eliminate the impact of kernel concentration, it extends the range of system sizes for which meaningful learning remains feasible.

The right panel examines the dependence of performance on the measurement budget for a fixed system size ($50$ qubits). For very small budgets, both methods perform similarly, as the available measurements are insufficient to reliably estimate the kernel structure. As the budget increases, a clear performance gap emerges: adaptive allocation significantly outperforms uniform sampling by focusing resources on the most informative kernel entries, leading to substantially lower decision-function error. In the large-budget regime, the gap closes as both methods approach the ideal classifier, with sufficient measurements to accurately reconstruct the kernel regardless of the allocation strategy.

Taken together, these results demonstrate that the effectiveness of adaptive measurement depends on both the structure of the kernel and the available measurement budget. While its benefits are limited in regimes where the learning problem lacks structure or where the signal is fundamentally degraded by concentration, adaptive allocation provides substantial gains in the practically relevant intermediate regime and remains effective over a broad range of budgets. In particular, it enables more efficient use of measurement resources and extends the operating regime of kernel-based learning beyond what is achievable with uniform sampling alone.

\subsection{Evaluation on Quantum Hardware}\label{sec:ibm_hardware_validation} 
To assess whether the proposed adaptive allocation strategy remains effective beyond simulated measurement models, we performed a hardware validation using \texttt{IBM Quantum Runtime} on the \texttt{IBM Marrakesh} backend. 
The experiment uses the synthetic binary classification dataset, consisting of $n=8$ training samples, for which the reference true-kernel SVM has two support vectors. 
Kernel entries are estimated using a direct shot-based fidelity-kernel procedure based on the compute--uncompute construction \cite{havlivcek2019supervised}. 
The quantum feature map consists of two repetitions of an angle-embedding layer followed by CNOT entangling gates, producing a nontrivial kernel while keeping the circuits sufficiently small for repeated hardware execution. 

Each experiment uses a fixed total measurement budget corresponding to $40$ shots per independent off-diagonal kernel entry (in total $1120$ shots per experiment). 
The adaptive procedure starts from a pilot allocation of $8$ shots per entry and then performs three adaptive rounds using multinomial allocation with $\lambda=0.5$. The SVM regularization parameter is fixed to $C=10$. We repeat the full adaptive-versus-uniform comparison over $15$ independent hardware runs.


Table~\ref{tab:ibm_hardware_summary} summarizes the final adaptive and uniform results. The hardware results reveal a clear separation between task-agnostic kernel reconstruction and SVM-relevant estimation quality. 
Uniform allocation achieves lower full-kernel reconstruction error, with full-kernel RMSE $0.0556 \pm 0.0086$, compared with $0.0771 \pm 0.0162$ for adaptive allocation. 
This behavior is expected: uniform allocation spreads measurements evenly across all kernel entries and is therefore well matched to global Gram matrix reconstruction. 
The dominance of uniform allocation in full-kernel reconstruction is also reflected in the paired comparisons. Uniform achieves lower full-kernel RMSE in 14 of the 15 hardware runs, yielding a statistically significant paired sign-test result.

In contrast, adaptive allocation substantially improves the quantities that determine the downstream SVM classifier. 
The RMSE restricted to the support-vector block decreases from $0.0314 \pm 0.0236$ under uniform allocation to $0.0112 \pm 0.0064$ under adaptive allocation, corresponding to an approximate $64\%$ reduction. 
Weighted support-vector recovery improves from $0.7919 \pm 0.1082$ to $0.9129 \pm 0.0539$, with adaptive allocation outperforming uniform allocation in $12$ of $15$ runs, with the paired sign test confirming that the observed improvement is unlikely to arise from random run-to-run variability alone. The classifier-level quantities show the same trend: the margin error decreases from $0.7632 \pm 0.3689$ to $0.2664 \pm 0.1610$, while the decision-function RMSE decreases from $0.4606 \pm 0.2130$ to $0.1659 \pm 0.0974$.
Adaptive allocation wins on both of these classifier-level metrics in $14$ of $15$ hardware runs. 

Figure~\ref{fig:ibm_hardware_stage_summary} provides a more detailed view of this trade-off. 
The top-left panel shows that adaptive allocation gradually sacrifices global kernel accuracy relative to the uniform baseline. 
In contrast, the remaining panels demonstrate consistent improvements in support-vector reconstruction, weighted support-vector agreement, and decision-function estimation throughout the adaptive rounds. 
These trends closely mirror the behavior observed in the synthetic experiments, indicating that the proposed strategy remains effective under realistic hardware noise.


We further evaluate the theoretical margin-variance objective using the reference dual coefficients, the reference kernel, and the actual final allocations produced in each hardware run. 
As shown in Table~\ref{tab:ibm_margin_variance}, the predicted margin variance under adaptive allocation is only $7.8\%$ of the variance induced by uniform allocation on average, with run-wise ratios ranging from $5.5\%$ to $10.8\%$. 
The corresponding oracle allocation achieves a ratio of $3.6\%$. Thus, even though the adaptive policy relies only on noisy sequential kernel estimates, it moves substantially toward the theoretically optimal allocation and explains the improved margin and decision-function estimates observed on hardware. 
\begin{table*}[t]
\centering
\caption{
IBM hardware results over 15 independent paired runs. Values are reported as mean \(\pm\) standard deviation. The relative adaptive effect is positive when adaptive allocation improves over uniform allocation; for error metrics this corresponds to relative error reduction, whereas for weighted Jaccard it corresponds to relative score increase. The final column reports two-sided paired sign-test \(p\)-values computed from the run-wise win counts, with ties excluded from the test.
}
\label{tab:ibm_hardware_summary}
\resizebox{\textwidth}{!}{%
\begin{tabular}{lccccc}
\toprule
Metric & Adaptive & Uniform & Rel. adaptive effect & Adaptive wins & Sign-test \(p\) \\
\midrule
Full-kernel RMSE & $0.0771 \pm 0.0162$ & $0.0556 \pm 0.0086$ & $-38.7\%$ & $1/15$ & $9.8e-04$ \\
SV-block RMSE & $0.0112 \pm 0.0064$ & $0.0314 \pm 0.0236$ & $64.4\%$ & $12/15$ & $3.5e-02$ \\
Weighted Jaccard & $0.9129 \pm 0.0539$ & $0.7919 \pm 0.1082$ & $15.3\%$ & $13/15$ & $7.4e-03$ \\
Margin error & $0.2664 \pm 0.1610$ & $0.7632 \pm 0.3689$ & $65.1\%$ & $14/15$ & $9.8e-04$ \\
Decision RMSE & $0.1659 \pm 0.0974$ & $0.4606 \pm 0.2130$ & $64.0\%$ & $14/15$ & $9.8e-04$ \\
\bottomrule
\end{tabular}%
}
\end{table*} 
\begin{table}[t]
\centering
\caption{
Predicted margin-variance ratios on IBM hardware, computed using the reference dual coefficients, the reference kernel, and the final allocations from each run. Ratios are reported relative to uniform allocation.
}
\label{tab:ibm_margin_variance}
\begin{tabular}{lcccc}
\toprule
Ratio & Mean & Std. & Min. & Max. \\
\midrule
Adaptive / uniform & $7.8\%$ & $1.3\%$ & $5.5\%$ & $10.8\%$ \\
Oracle / uniform & $3.6\%$ & $0.0\%$ & $3.6\%$ & $3.6\%$ \\
\bottomrule
\end{tabular}
\end{table}

\begin{figure*}[t] 
\centering 
\begin{subfigure}[t]{0.48\textwidth} 
\centering \includegraphics[width=\textwidth]{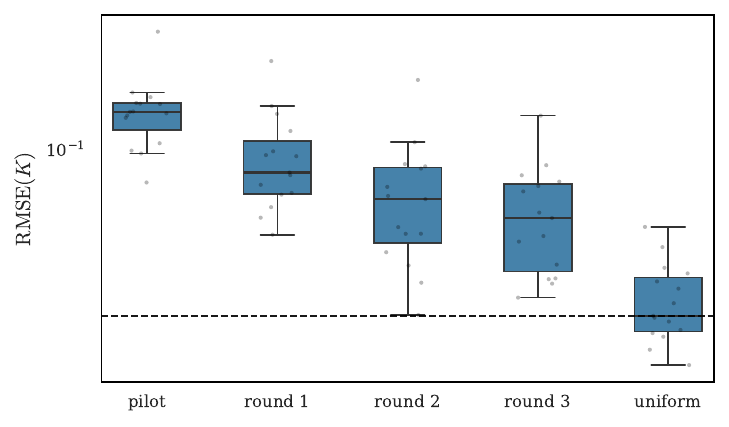} 
\caption{Global kernel error \(\mathrm{RMSE}(K)\).} \label{fig:ibm_rmse_k} 
\end{subfigure} 
\hfill 
\begin{subfigure}[t]{0.48\textwidth} 
\centering \includegraphics[width=\textwidth]{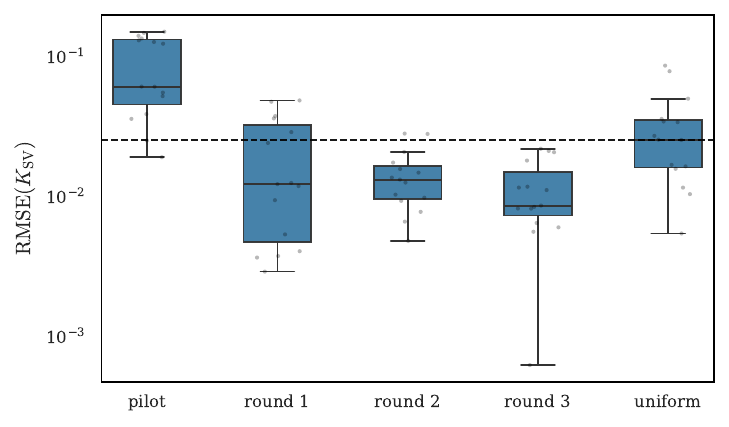} 
\caption{SV-block error \(\mathrm{RMSE}(K_{\mathrm{SV}})\).} \label{fig:ibm_rmse_k_sv} 
\end{subfigure} 
\vspace{0.6em}
\begin{subfigure}[t]{0.48\textwidth} 
\centering \includegraphics[width=\textwidth]{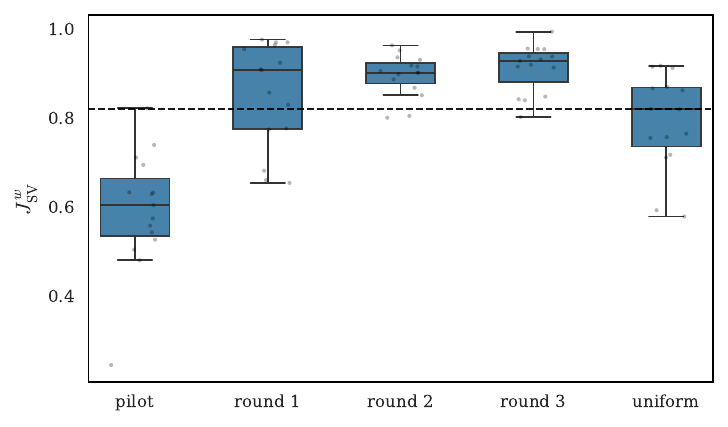} 
\caption{Dual-weighted support-vector agreement \(J_{\mathrm{SV}}^{w}\).} \label{fig:ibm_jaccard_weighted} 
\end{subfigure} 
\hfill 
\begin{subfigure}[t]{0.48\textwidth} 
\centering \includegraphics[width=\textwidth]{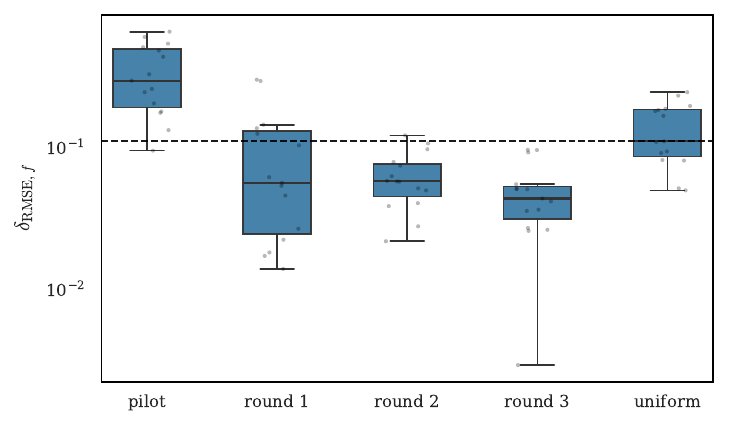} 
\caption{Normalized decision-function error \(\delta_{\mathrm{RMSE},f}\).} \label{fig:ibm_func_rmse_normalized} 
\end{subfigure}
\caption{ IBM Quantum hardware validation across the pilot stage, adaptive rounds, and uniform baseline. Horizontal dashed lines indicate the median performance of the uniform baseline. Uniform allocation achieves the best global kernel reconstruction, whereas adaptive allocation progressively improves the SVM-relevant quantities: support-vector-block reconstruction, weighted support-vector recovery, and decision-function estimation. } \label{fig:ibm_hardware_stage_summary} \end{figure*}

\section{Software Release and Reproducibility}\label{sec:shotwise}
To simplify reproducibility and adoption, we release an open-source Python package, \texttt{ShotWise}, implementing the adaptive measurement-allocation framework developed in this work. The package provides the SupportShot algorithm introduced in this paper, supports synthetic measurement backends, \texttt{Qiskit} simulators, and IBM Quantum hardware, and includes utilities for experiment management and statistical analysis. The implementation follows a backend-agnostic design, enabling future extensions to additional measurement-based kernel-learning settings beyond quantum kernels. The repository and documentation are available at: \url{https://github.com/ESA-PhiLab/shotwise}

\section{Conclusion}\label{sec:conclusion}

In this work, we studied the problem of learning kernelized Support Vector Machines under measurement constraints, where the kernel matrix must be estimated from noisy observations and the learner controls accuracy through the allocation of a limited measurement budget. We showed that uniform allocation, while statistically balanced, fails to exploit the highly non-uniform dependence of the SVM solution on the Gram matrix.

To address this, we introduced an adaptive measurement-allocation strategy that is explicitly guided by the structure of the SVM problem. Our approach combines geometric sensitivity, capturing the influence of individual kernel entries on the margin, with active set instability, quantifying the probability of discrete changes in support-vector membership. These complementary signals enable a task-aware allocation scheme that concentrates measurement effort on kernel matrix entries most relevant to the learned classifier.

On the theoretical side, we demonstrated that the advantage of adaptive allocation is governed by the heterogeneity of the induced kernel importance structure. This leads to a clear regime-dependent picture: in strongly structured problems, adaptive allocation provides substantial gains, while in near-homogeneous settings its benefits diminish due to estimation noise. On the algorithmic side, we developed a multi-round adaptive procedure with an early-stopping criterion based on dual coefficient stability, allowing the method to terminate once the classifier has effectively converged. Experimentally, we showed that this approach improves support-vector recovery, margin estimation, and decision-function accuracy under fixed budgets, while often requiring only a fraction of the measurement budget.

Beyond synthetic evaluations, experiments on quantum kernels derived from real-world data revealed a consistent interplay between kernel structure, measurement budget, and system size. In particular, adaptive allocation is most effective in intermediate regimes where the kernel induces a structured and heterogeneous SVM solution, leading to substantial improvements in decision-function accuracy. While performance degrades in regimes dominated by kernel concentration or severe measurement noise, adaptive allocation continues to make more efficient use of the available measurement budget and extends the range of conditions under which meaningful learning remains feasible. Importantly, these trends were observed not only in simulation but also in experiments conducted on \texttt{IBM Quantum} hardware, where the proposed adaptive strategy remained effective despite device noise. This connection highlights the practical relevance of the proposed method for quantum machine learning, where measurement efficiency, hardware noise, and kernel degradation are central challenges.

Several directions for future work remain open. On the theoretical side, a deeper understanding of support-vector transitions and active-set instability could help reduce the gap between local sensitivity analysis and global changes in classifier structure. More broadly, the present work suggests a general perspective in which finite measurement resources are allocated according to their downstream impact on learning. Extending this principle beyond kernelized SVMs to other kernel methods and statistical learning algorithms is therefore a natural next step. Finally, integrating adaptive measurement allocation with complementary scalability techniques, including low-rank kernel approximations and matrix-compression methods, may enable efficient learning from large noisy kernel matrices under realistic resource constraints.

\section*{Acknowledgements} 

The author would like to thank Lorenzo Papa for his support in obtaining access to IBM Quantum devices and Jian Xu for valuable discussions related to this work. 

This work was conducted within the European Space Agency (ESA). ESA classification: UNCLASSIFIED -- Releasable to the Public.
\bibliographystyle{IEEEtran}
\bibliography{ref_all}

\begin{appendices}

\section{Variance of Kernel Estimates under Hardware and Sampling Noise}\label{App:variance}

In this appendix we derive the expectation and variance of the empirical kernel estimator in the presence of both sampling noise and stochastic hardware-induced fluctuations.

\subsection{Model Definition}

We consider a fixed kernel entry $(i,j)$ and omit indices for clarity. Let
\begin{equation}
    K \in [0,1]
\end{equation}
denote the true kernel value.

We assume that each execution of the kernel evaluation procedure produces an \emph{effective} kernel realization
\begin{equation}
    \tilde{K}^{(t)} = K + \varepsilon^{(t)},
\end{equation}
where the noise process $\{\varepsilon^{(t)}\}_{t=1}^N$ satisfies
\begin{equation}
    \mathbb{E}[\varepsilon^{(t)}] = 0,
\end{equation}
and is not necessarily independent across different measurements. In particular, we allow
\begin{equation}
    \mathrm{Cov}(\varepsilon^{(t)}, \varepsilon^{(s)}) \neq 0 \quad \text{for } t \neq s.
\end{equation}

Given a realization $\tilde{K}^{(t)}$, the observed measurement is
\begin{equation}
    X^{(t)} \sim \mathrm{Bernoulli}(\tilde{K}^{(t)}),
\end{equation}
and we define the empirical estimator
\begin{equation}
    \widehat{K} = \frac{1}{N} \sum_{t=1}^N X^{(t)}.
\end{equation}

\subsection{Expectation of the Estimator}

We compute the expectation using the law of total expectation:
\begin{align}
    \mathbb{E}[X^{(t)}]
    &= \mathbb{E}\big[\mathbb{E}[X^{(t)} \mid \tilde{K}^{(t)}]\big] \\
    &= \mathbb{E}\big[\tilde{K}^{(t)}\big] \\
    &= K.
\end{align}
Therefore,
\begin{equation}
    \mathbb{E}[\widehat{K}] = \frac{1}{N} \sum_{t=1}^{N} \mathbb{E}[X^{(t)}] = K.
\end{equation}

Thus, the estimator $\widehat{K}$ is unbiased for the true kernel value.

\subsection{Variance of the Estimator}

We compute the variance of $\widehat{K}$:
\begin{equation}
    \mathrm{Var}(\widehat{K})
    =
    \mathrm{Var}\left(\frac{1}{N} \sum_{t=1}^{N} X^{(t)}\right).
\end{equation}

Using the standard variance decomposition for sums,
\begin{equation}
    \mathrm{Var}(\widehat{K})
    =
    \frac{1}{N^2}
    \left(
        \sum_{t=1}^{N} \mathrm{Var}(X^{(t)})
        +
        \sum_{t \neq s} \mathrm{Cov}(X^{(t)}, X^{(s)})
    \right).
\end{equation}

\subsubsection{Variance of Individual Measurements}

We apply the law of total variance:
\begin{align}
    \mathrm{Var}(X^{(t)})
    &=
    \mathbb{E}[\mathrm{Var}(X^{(t)} \mid \tilde{K}^{(t)})]
    +
    \mathrm{Var}(\mathbb{E}[X^{(t)} \mid \tilde{K}^{(t)}]).
\end{align}

Since
\begin{align}
    \mathrm{Var}(X^{(t)} \mid \tilde{K}^{(t)}) &= \tilde{K}^{(t)}(1 - \tilde{K}^{(t)}), \\
    \mathbb{E}[X^{(t)} \mid \tilde{K}^{(t)}] &= \tilde{K}^{(t)},
\end{align}
we obtain
\begin{equation}
    \mathrm{Var}(X^{(t)})
    =
    \mathbb{E}\big[\tilde{K}^{(t)}(1 - \tilde{K}^{(t)})\big]
    +
    \mathrm{Var}(\tilde{K}^{(t)}).
\end{equation}

Using the identity
\begin{equation}
    \mathbb{E}[\tilde{K}(1-\tilde{K})] = K(1-K) - \mathrm{Var}(\tilde{K}),
\end{equation}
we obtain
\begin{equation}
    \mathrm{Var}(X^{(t)}) = K(1-K).
\end{equation}

\subsubsection{Covariance Between Measurements}

We compute the covariance using the law of total covariance:
\begin{equation} 
\begin{split} 
\mathrm{Cov}(X^{(t)}, X^{(s)}) &= \mathbb{E}[\mathrm{Cov}(X^{(t)}, X^{(s)} \mid \tilde{K})] \\ 
&\quad + \mathrm{Cov}(\mathbb{E}[X^{(t)} \mid \tilde{K}], \mathbb{E}[X^{(s)} \mid \tilde{K}]). 
\end{split} 
\end{equation}

Given $\tilde{K}^{(t)}$ and $\tilde{K}^{(s)}$, the measurements are conditionally independent, hence
\begin{equation}
    \mathrm{Cov}(X^{(t)}, X^{(s)} \mid \tilde{K}) = 0.
\end{equation}

Therefore,
\begin{equation}
    \mathrm{Cov}(X^{(t)}, X^{(s)}) =
    \mathrm{Cov}(\tilde{K}^{(t)}, \tilde{K}^{(s)}).
\end{equation}

We define
\begin{equation}
    \sigma^2_{\mathrm{phys}} :=
    \mathrm{Cov}(\tilde{K}^{(t)}, \tilde{K}^{(s)}),
    \quad t \neq s,
\end{equation}
which captures the contribution of correlated hardware-induced fluctuations.

\subsubsection{Final Expression}

We now combine the results:
\begin{align}
    \mathrm{Var}(\widehat{K})
    &=
    \frac{1}{N^2}
    \left(
        N K(1-K)
        +
        N(N-1)\sigma^2_{\mathrm{phys}}
    \right) \\
    &=
    \frac{K(1-K)}{N}
    +
    \left(1 - \frac{1}{N}\right)\sigma^2_{\mathrm{phys}}.
\end{align}

\subsection{Discussion}

The variance of the kernel estimator decomposes into two contributions:
\begin{itemize}
    \item A \emph{sampling (measurement) noise} term, $\frac{K(1-K)}{N}$, which decreases with the number of measurements.
    \item An \emph{irreducible hardware-induced term}, 
    $\left(1 - \frac{1}{N}\right)\sigma^2_{\mathrm{phys}}$, which arises from correlations between kernel realizations across measurements.
\end{itemize}

In particular,
\begin{equation}
    \lim_{N \to \infty} \mathrm{Var}(\widehat{K}) = \sigma^2_{\mathrm{phys}},
\end{equation}
showing that correlated fluctuations of the effective kernel impose a fundamental limit on the precision of kernel estimation that cannot be overcome by increasing the number of measurements.

\section{Margin Sensitivity via the Envelope Theorem}\label{app:envelope}
In Sec.~\ref{sec:theory_sensitivity} we use the fact that, when a perturbation of a kernel entry $K_{ij}$ does not alter the active set of support vectors, the derivative of the squared margin with respect to $K_{ij}$ can be computed without differentiating through the optimal dual variables $\alpha^*(K)$.
This appendix provides a concise justification based on the envelope theorem.

Recall the SVM dual objective:
\begin{equation}
    \max_{\alpha\in\mathcal{A}}\ g(\alpha,K),
\end{equation}
with
\begin{equation}
    g(\alpha,K)=
\sum_{i=1}^n \alpha_i
-
\frac12 \sum_{i,j=1}^n \alpha_i\alpha_j y_i y_j K_{ij},
\end{equation}
and feasible set
$$\mathcal{A}
=
\left\{
0\le\alpha_i\le C,\;
\sum_i \alpha_i y_i = 0
\right\}.$$
For fixed kernel $K$, let $\alpha^*(K)$ denote the unique optimal solution (on the active set manifold).
Define the optimal dual value:
\begin{equation}
V(K)=g(\alpha^*(K),K).    
\end{equation}
The (Milgrom–Segal) envelope theorem states that,
if $V(K) = \max_{\alpha \in \mathcal{A}} g(\alpha,K)$ and the maximizer $\alpha^*(K)$ varies smoothly in a neighborhood where the active set of constraints does not change,
then
\begin{equation}
    \frac{\partial V(K)}{\partial K_{ij}}
=
\frac{\partial g(\alpha, K)}{\partial K_{ij}}
\bigg|_{\alpha = \alpha^*(K)}.
\end{equation}
Intuitively, when $\alpha^*$ is the optimizer for the current $K$, the directional derivative of $g$ in any feasible tangent direction is already zero (first-order KKT stationarity). Therefore the only ``direct" sensitivity at optimum comes from the explicit dependence of the objective on $K$, not from the optimizer’s movement.

\section{Optimal Measurement Allocation}\label{app:optimization}

\subsection{General Allocation Principle}

Consider the optimization problem
\begin{equation}
\min_{\{N_{ij}\}} \sum_{i<j} \frac{w^2_{ij}}{N_{ij}}
\quad
\text{s.t.}
\quad
\sum_{i<j} N_{ij} = N_{\mathrm{tot}}, \quad N_{ij} > 0,
\end{equation}
where $w_{ij} \ge 0$ are arbitrary weights.

The Lagrangian is given by
\begin{equation}
\mathcal{L} =
\sum_{i<j} \frac{w^2_{ij}}{N_{ij}}
+
\lambda \left( \sum_{i<j} N_{ij} - N_{\mathrm{tot}} \right).
\end{equation}

Differentiating with respect to $N_{ij}$ yields
\begin{equation}
-\frac{w_{ij}^2}{N_{ij}^2} + \lambda = 0,
\end{equation}
which implies
\begin{equation}
N_{ij} = \sqrt{\frac{w_{ij}^2}{\lambda}}.
\end{equation}

Imposing the constraint gives
\begin{equation}
\sum_{i<j} \sqrt{\frac{w^2_{ij}}{\lambda}} = N_{\mathrm{tot}},
\end{equation}
hence
\begin{equation}
\sqrt{\lambda} = \frac{\sum_{i<j} \sqrt{w^2_{ij}}}{N_{\mathrm{tot}}}.
\end{equation}

Substituting back, we obtain the general solution
\begin{equation}
\boxed{
N_{ij}^\star
=
N_{\mathrm{tot}}
\frac{w_{ij}}{\sum_{k<l} w_{kl}}.
}
\end{equation}

\subsection{Margin Variance Minimization}

Using first-order sensitivity analysis, the variance of the squared margin can be approximated as
\begin{equation}
V_{\mathrm{margin}}
=
\sum_{i<j}
(\alpha_i \alpha_j)^2
\frac{K_{ij}(1-K_{ij})}{N_{ij}}.
\end{equation}

This corresponds to weights
\begin{equation}
\left(w_{ij}^{\mathrm{margin}}\right)^2
\coloneqq
(\alpha_i \alpha_j)^2 K_{ij}(1-K_{ij}).
\end{equation}

Substituting into the general solution yields
\begin{equation}
\boxed{
N_{ij}^{\mathrm{margin}}
\propto
|\alpha_i \alpha_j|\;\sqrt{K_{ij}(1-K_{ij})}.
}
\end{equation}

\subsection{Decision Function Variance Minimization}

The variance of the decision function is given by
\begin{equation}
\sigma^2_{f,i}
=
\sum_j (\alpha_j y_j)^2 \frac{K_{ij}(1-K_{ij})}{N_{ij}}.
\end{equation}

Summing over $i$ yields
\begin{equation}
V_{\mathrm{dec}} =
\sum_{i<j}
(\alpha_i^2 + \alpha_j^2)
\frac{K_{ij}(1-K_{ij})}{N_{ij}},
\end{equation}
which corresponds to weights
\begin{equation}
\left( w_{ij}^{\mathrm{dec}}\right)^2
\coloneqq
(\alpha_i^2 + \alpha_j^2)\,K_{ij}(1-K_{ij}).
\end{equation}

Applying the general solution gives
\begin{equation}
\boxed{
N_{ij}^{\mathrm{dec}}
\propto
\sqrt{\alpha_i^2 + \alpha_j^2}\;\sqrt{K_{ij}(1-K_{ij})}.
}
\end{equation}

\subsection{Comparison of Allocation Strategies}

Both allocation schemes share the same functional form and differ only in the dependence on the dual variables:
\begin{align}
N_{ij}^{\mathrm{margin}} 
&\propto |\alpha_i \alpha_j|, \\
N_{ij}^{\mathrm{dec}} 
&\propto \sqrt{\alpha_i^2 + \alpha_j^2}.
\end{align}

Using the inequalities
\begin{equation}
|\alpha_i \alpha_j|
\le \frac{\alpha_i^2 + \alpha_j^2}{2},
\end{equation}
it follows that the two allocation strategies are equivalent up to multiplicative constants.

In particular, both concentrate measurement effort on pairs involving data points with large dual coefficients, i.e., on support vectors.

\section{Proposition proofs}\label{app:propositions}
\subsection{Proposition 1}
The proof follows from the obtained variances. The uniform approach gives
\[
V_{\mathrm{unif}} = \frac{n(n-1)}{2N_{\mathrm{tot}}} \sum_{i<j} w^2_{ij},
\]
while the optimal allocation yields
\[
V^\star = \frac{(\sum_{i<j} w_{ij})^2}{N_{\mathrm{tot}}}.
\]
and the Cauchy-Schwarz inequality,
\[
{\biggl (}\sum _{i=1}u_{i}v_{i}{\biggr )}^{2}\leq {\biggl (}\sum _{i=1}u_{i}^{2}{\biggr )}{\biggl (}\sum _{i=1}v_{i}^{2}{\biggr )}.
\]
We take the sum over all the independent matrix elements $\sum_{i<j} \equiv \sum_{i}$, take $u_i = w_{ij}$, and $v_i = 1$. We get
\[
{\biggl (}\sum _{i<j}w_{ij} \cdot 1{\biggr )}^{2}\leq {\biggl (}\sum _{i=1}w^2_{ij}{\biggr )}\underbrace{{\biggl (}\sum _{i<j}1{\biggr )}}_{\frac{n(n-1)}{2}}.
\]
Dividing both sides of the equation by $N_{tot}$, we get
\[
V^* \leq V_{unif},
\]
which finishes the proof.
\subsection{Proposition 2}
The variance objective is given by
\[
V = \sum_{i<j} \frac{w^2_{ij}}{N_{ij}}.
\]

Let $N_{ij}^\star$ denote the oracle-optimal allocation, and consider a perturbed allocation
\[
N_{ij} = N_{ij}^\star + \delta N_{ij},
\]
subject to the constraint
\[
\sum_{i<j} \delta N_{ij} = 0.
\]

We perform a second-order Taylor expansion of the function $f(x) = 1/x$ around $x = N_{ij}^\star$:
\[
\frac{1}{N_{ij}} =
\frac{1}{N_{ij}^\star}
- \frac{\delta N_{ij}}{(N_{ij}^\star)^2}
+ \frac{(\delta N_{ij})^2}{(N_{ij}^\star)^3}
+ \mathcal{O}((\delta N_{ij})^3).
\]

Substituting this expansion into the expression for $V$, we obtain
\[
V \approx
\sum_{i<j} w^2_{ij}
\left(
\frac{1}{N_{ij}^\star}
- \frac{\delta N_{ij}}{(N_{ij}^\star)^2}
+ \frac{(\delta N_{ij})^2}{(N_{ij}^\star)^3}
\right).
\]

The first term recovers the optimal value:
\[
V^\star = \sum_{i<j} \frac{w^2_{ij}}{N_{ij}^\star}.
\]

The first-order term vanishes at the optimum. This follows from the first-order optimality (KKT) conditions for the constrained minimization problem, which ensure that the gradient of $V$ at $N^\star$ is orthogonal to feasible perturbations satisfying $\sum_{i<j} \delta N_{ij} = 0$.

Thus, retaining terms up to second order,
\[
V \approx V^\star + \sum_{i<j} \frac{w^2_{ij}}{(N_{ij}^\star)^3} (\delta N_{ij})^2.
\]

Taking expectation and assuming that $\mathbb{E}[\delta N_{ij}] = 0$, we obtain
\[
\mathbb{E}[V]
=
V^\star
+
\sum_{i<j}
\frac{w^2_{ij}}{(N_{ij}^\star)^3}
\, \mathbb{E}\left[(\delta N_{ij})^2\right]
+
\mathcal{O}(\|\delta N\|^3),
\]
which completes the proof.

\subsection{Proposition 3} \label{app:proof_missing_sensitivity} 
Let \(Z_i\in\{0,1\}\) denote the event that point \(i\) belongs to the perturbation-induced margin-active set. 
Let \(\alpha_i^{new}\) denote the dual coefficient after the perturbation, and define the transition-induced sensitivity associated with entry \((i,j)\) as \[ h_{ij} := \alpha_i^{new}\alpha_j^{new} y_i y_j\, Z_iZ_j . \] 
This quantity represents the possible margin sensitivity that is not captured by the local fixed-active-set approximation but may appear after an active set transition. 
By the box constraints of the soft-margin SVM, \[ 0 \leq \alpha_i^{new} \leq C \qquad \text{for all } i. \] 
Therefore, whenever \(Z_iZ_j=1\), \[ |h_{ij}| = |\alpha_i^{new}\alpha_j^{new} y_i y_j| \leq C^2, \] since \(|y_i y_j|=1\). 
If \(Z_iZ_j=0\), then \(h_{ij}=0\) by definition. 
Hence, in all cases, \[ |h_{ij}| \leq C^2 Z_iZ_j . \] 
Taking expectations gives \[ \mathbb E[|h_{ij}|] \leq C^2 \mathbb E[Z_iZ_j] = C^2 \Pr(Z_i=1,Z_j=1). \] 
The instability probability \(P_i\) is used as a proxy for the probability that point \(i\) becomes margin-active under measurement noise. 
Under the weak-dependence approximation \[ \Pr(Z_i=1,Z_j=1) \lesssim P_iP_j, \] we obtain \[ \mathbb E[|h_{ij}|] \lesssim C^2P_iP_j . \] 
This proves the claim.

\begin{figure*}[t] 
\centering 
\includegraphics[width=0.8  \linewidth]{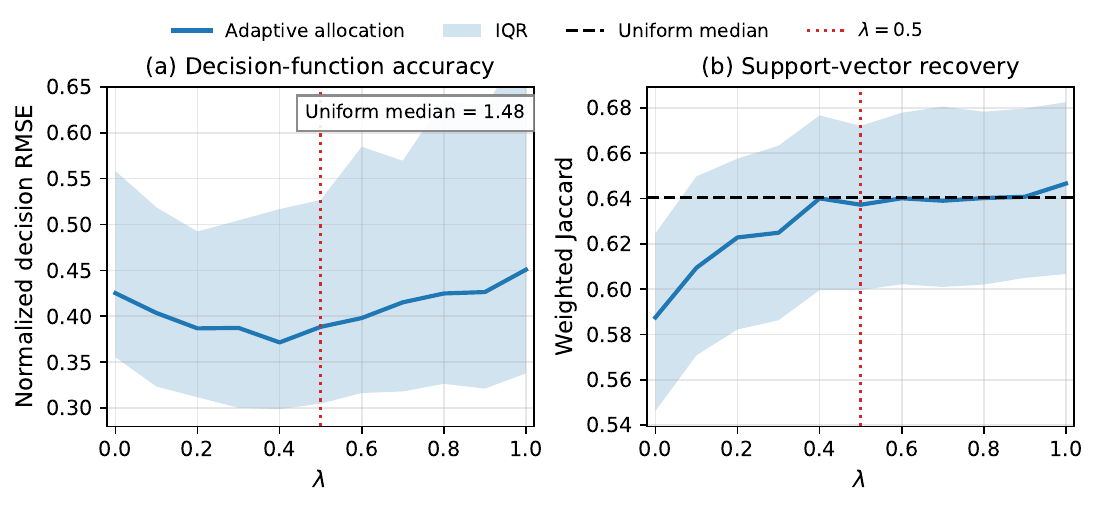} 
\caption{ Ablation study of the allocation mixing parameter $\lambda$ under a fixed measurement budget. The parameter interpolates between margin-sensitivity allocation ($\lambda=0$) and active set instability allocation ($\lambda=1$). Panel (a) reports the normalized decision-function RMSE with respect to the true-kernel SVM, while panel (b) reports the weighted Jaccard similarity between the estimated and reference dual coefficient vectors. Solid lines show the median over 500 independent runs and shaded regions denote the interquartile range. The vertical red line indicates the default choice $\lambda=0.5$. Best decision-function accuracy is obtained for intermediate values of $\lambda$, whereas weighted Jaccard generally increases as more weight is placed on active set instability. These results indicate that margin sensitivity and active set instability provide complementary information for adaptive measurement allocation. } 
\label{fig:lambda_ablation} 
\end{figure*}

\section{Consistency of Adaptive Allocation} \label{app:consistency_limit}
\subsection{Proof of Theorem~\ref{thm:allocation_consistency}}
Throughout this section, the notation \[ \hat K^{(r)} \xrightarrow{p} K \] denotes convergence in probability. Specifically, as the number of measurements allocated to each independent kernel entry increases, \[ N_{ij}\to\infty, \qquad i<j, \] we have \[ \lim_{\min_{i<j}N_{ij}\to\infty} \Pr\!\left( \|\hat K^{(r)}-K\|>\varepsilon \right) = 0 \] for every $\varepsilon>0$. 
More generally, for any sequence of estimators $\hat\theta^{(r)}$ and parameter $\theta$, the notation \[ \hat\theta^{(r)} \xrightarrow{p} \theta \] means that the probability of observing a deviation larger than any fixed tolerance $\varepsilon>0$ converges to zero as the measurement budget increases.
\begin{proof} 
By assumption, the kernel estimator is consistent, and therefore \[ \hat K_{ij}^{(r)} \xrightarrow{p} K_{ij} \qquad \text{for all } i<j. \] 
Equivalently, \[ \hat K^{(r)} \xrightarrow{p} K. \] 
Under uniqueness of the SVM dual optimum and the Local Active Set Stability assumption, Assumption~\ref{ass:active}, sufficiently small perturbations of the kernel matrix do not alter the support-vector partition. 
Standard sensitivity results for strictly convex quadratic programs therefore imply continuity of the optimal dual solution with respect to the kernel matrix \cite{bonnans2013perturbation}. 
Consequently, \[ \hat \alpha^{(r)} \xrightarrow{p} \alpha. \] 
The allocation weights are constructed from the estimated kernel matrix and the corresponding estimated dual coefficients. 
Since the allocation rule is continuous by construction, the Continuous Mapping Theorem \cite{wasserman2004all} yields \[ \hat w_{ij}^{(r)} \xrightarrow{p} w_{ij} \qquad \text{for all } i<j. \] 
To establish convergence of the allocation itself, define \[ W = \sum_{k<l} w_{kl}, \qquad \hat W^{(r)} = \sum_{k<l}\hat w_{kl}^{(r)}. \] 
Since each individual weight converges in probability, \[ \hat W^{(r)} \xrightarrow{p} W. \] 
Provided that $W>0$, another application of the Continuous Mapping Theorem gives \[ \frac{\hat w_{ij}^{(r)}} {\hat W^{(r)}} \xrightarrow{p} \frac{w_{ij}} {W}. \] 
Therefore the allocation proportions induced by the adaptive procedure converge to the corresponding oracle Neyman allocation. This establishes asymptotic recovery of the oracle allocation. \end{proof}

The theorem implies that the adaptive procedure asymptotically recovers the oracle measurement-allocation strategy derived from the true kernel matrix and the corresponding optimal SVM solution. Consequently, the adaptive algorithm may be interpreted as a statistically consistent estimator of the oracle allocation.

\subsection{Proof of Corollary~\ref{cor:efficiency}}
\begin{proof} 
Theorem~\ref{thm:allocation_consistency} establishes that \[ \hat w_{ij}^{(r)} \xrightarrow{p} w_{ij} \] for all independent kernel entries. 
Consequently, the allocation proportions induced by the adaptive procedure converge to the corresponding oracle Neyman allocation. 
Since the variance proxy of Eq.~\eqref{eq:margin_min} is a continuous function of the allocation proportions whenever all allocations are strictly positive, the Continuous Mapping Theorem implies \[ \hat V^{(r)} \xrightarrow{p} V^\star. \] 
Dividing by the nonzero constant \(V^\star\) gives \[ \frac{\hat V^{(r)}}{V^\star} \xrightarrow{p} 1, \] which establishes asymptotic oracle efficiency.
\end{proof}

\section{Sensitivity to the Allocation Mixing Parameter}\label{app.ablation_lambda}

To investigate the influence of the allocation mixing parameter, we evaluated the adaptive strategy across values of $\lambda \in [0,1]$ while keeping the total measurement budget fixed. Figure~\ref{fig:lambda_ablation} shows the resulting classifier reconstruction accuracy and support-vector recovery performance. For normalized decision-function RMSE, the best performance is obtained for intermediate values of $\lambda$, with a broad optimum around $\lambda \approx 0.4$. In contrast, weighted Jaccard similarity improves steadily as more emphasis is placed on active set instability. Neither extreme strategy performs best across both metrics: pure margin sensitivity ($\lambda=0$) yields inferior support-vector recovery, while pure active set instability ($\lambda=1$) degrades decision-function accuracy. This behavior supports the central design principle of the proposed method, namely that geometric sensitivity and active set uncertainty capture complementary aspects of kernel importance. The default choice $\lambda=0.5$ lies within a broad high-performing region and does not require precise tuning.

\section{Hyperparameter Selection}\label{sec:hyperparameters}

We summarize the key parameters used in the adaptive measurement procedure.

\begin{itemize}

    \item \textbf{Total measurement budget $N_{\mathrm{tot}}$.}  
    Determined by experimental constraints.

    \item \textbf{Number of adaptive rounds $R$.}  
    Sets the maximum number of refinement iterations. Unless early stopping is triggered, the remaining budget after the pilot stage is distributed uniformly across the rounds.

    \item \textbf{Score weight $\lambda$.}  
    Controls the balance between geometric sensitivity and active set instability
    (Eq.~\eqref{eq:sens_inst}). We use $\lambda = 0.5$.

    \item \textbf{Pilot measurements $m_0$.}  
    Each kernel entry is initialized using $m_0$ measurements, with total cost
    \[
        N_{\mathrm{pilot}} = m_0 \cdot \frac{n(n-1)}{2}.
    \]

    \item \textbf{Stopping threshold $\varepsilon$.}  
    The algorithm stops when $\delta_r < \varepsilon$. We consider values in the range $10^{-2}$–$10^{-1}$.

    \item \textbf{SVM regularization $C$.}  
    Fixed throughout optimization; default values from standard libraries are used.

\end{itemize}

\section{Additional IBM Hardware Analysis}\label{app:ibm}
\begin{figure*}[t] 
\centering 
\begin{subfigure}[t]{0.34\textwidth} \centering \includegraphics[width=\textwidth]{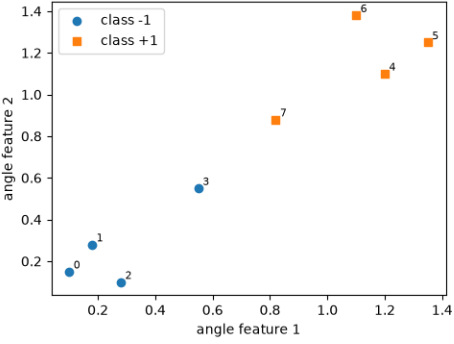} \caption{The toy dataset.} \label{fig:appendix_ibm_uniform_alloc} \end{subfigure} 
\hfill 
\begin{subfigure}[t]{0.3\textwidth} \centering \includegraphics[width=\textwidth]{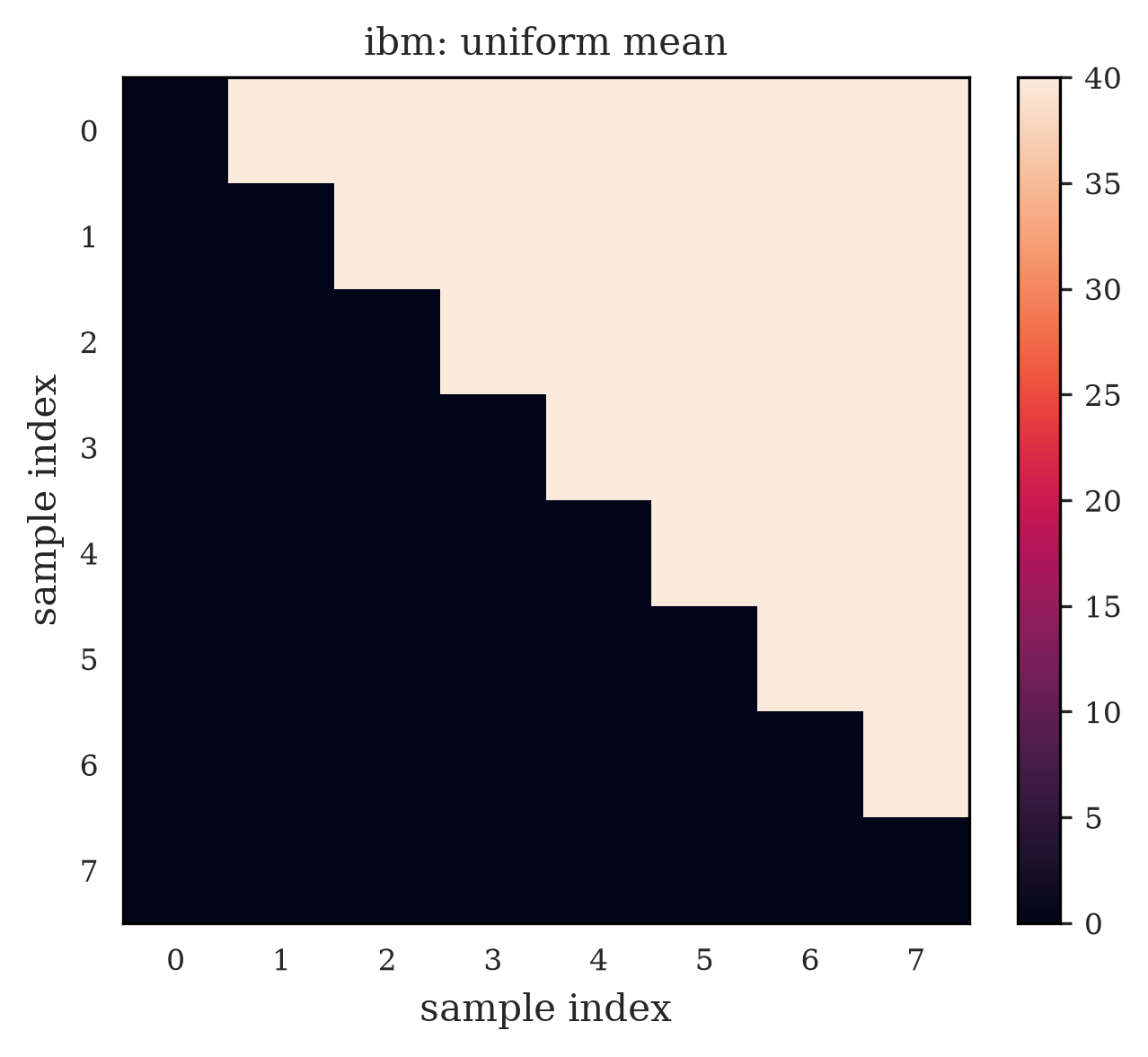} \caption{Mean uniform allocation.} \label{fig:appendix_ibm_adaptive_alloc} \end{subfigure} 
\hfill 
\begin{subfigure}[t]{0.3\textwidth} \centering \includegraphics[width=\textwidth]{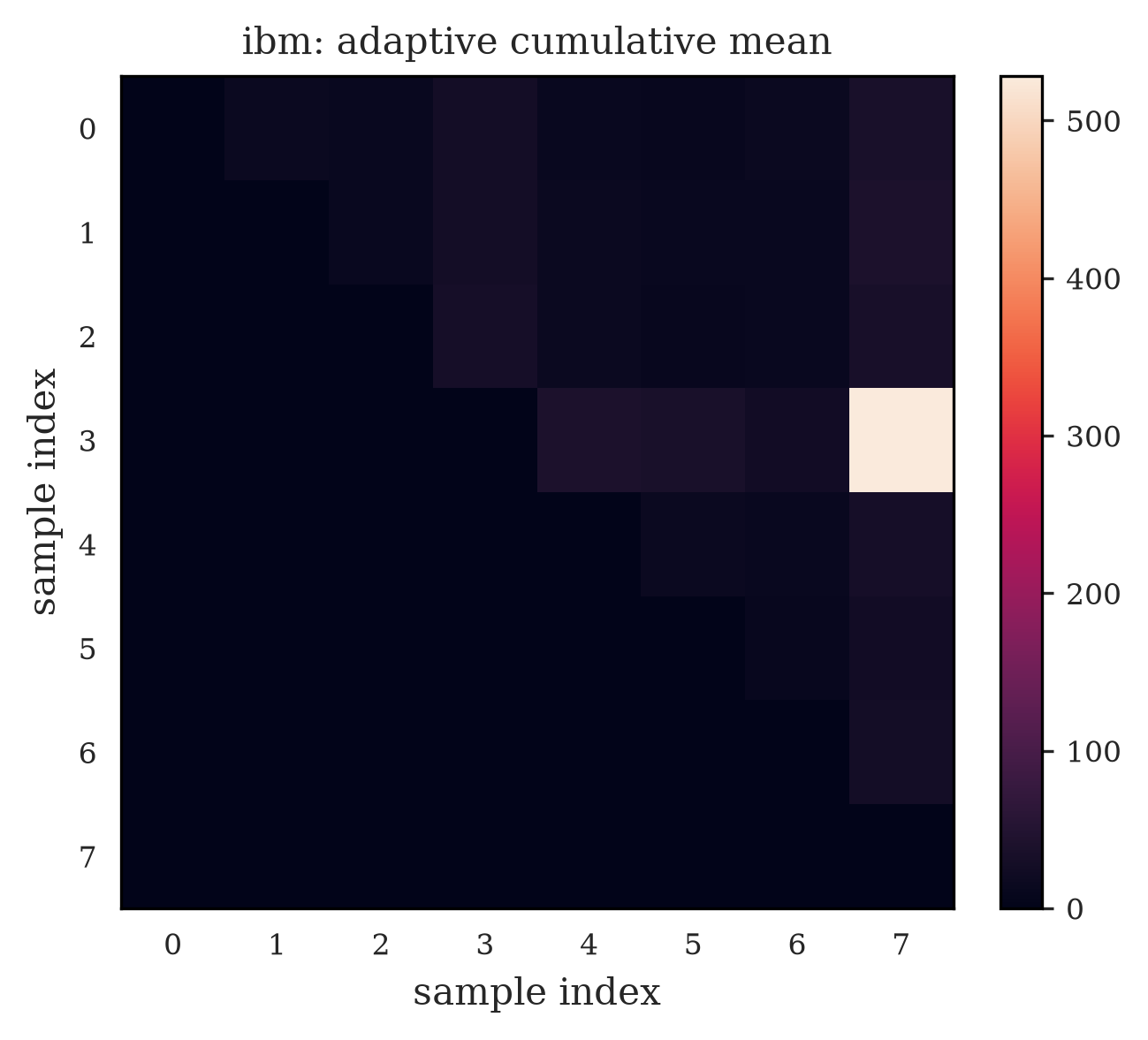} \caption{Mean adaptive allocation.} \label{fig:appendix_ibm_adaptive_minus_uniform} \end{subfigure} 
\caption{Additional hardware-allocation analysis for the IBM Quantum experiment. (a) Synthetic dataset used for hardware validation. Samples are labeled according to their index in the kernel matrix. (b) Mean uniform allocation across the 15 hardware runs. (c) Mean cumulative adaptive allocation across the same runs. The adaptive strategy concentrates measurements on a small subset of kernel entries rather than distributing them uniformly across the Gram matrix. In particular, a large fraction of the additional measurement budget is directed toward kernel entry $(3,7)$, corresponding to the pair of points closest to the decision boundary and identified as support vectors by the reference true-kernel SVM. This behavior is consistent with the sensitivity analysis of Sec.~\ref{sec:theory_sensitivity}, which predicts that kernel entries associated with large dual coefficients have the greatest influence on the margin and decision function. The allocation patterns therefore provide a mechanistic explanation for the improvements in support-vector recovery and classifier estimation observed in the hardware experiments.} \label{fig:appendix_ibm_allocations} 
\end{figure*}

The IBM Quantum hardware experiments reported in Sec.~\ref{sec:ibm_hardware_validation} were performed on the \texttt{IBM Marrakesh} backend using \texttt{Qiskit Runtime}. 
The objective of these experiments was not to demonstrate large-scale quantum advantage, but rather to validate the adaptive measurement allocation strategy under realistic hardware conditions, including finite-shot noise, device calibration variability, and execution-time constraints. 

The experiments used the synthetic binary classification dataset shown in Fig.~\ref{fig:appendix_ibm_allocations}. 
The dataset consists of $n=8$ training samples and yields a reference true-kernel SVM with two support vectors (samples $3$ and $7$). 
Throughout the experiments, the SVM regularization parameter was fixed to $C=10$. 

Kernel entries were estimated using a shot-based fidelity-kernel construction based on the compute--uncompute procedure of Ref.~\cite{havlivcek2019supervised}. 
The quantum feature map consisted of two repetitions of angle embedding followed by nearest-neighbour CNOT entangling gates. 
This choice produced a non-trivial kernel structure while keeping circuit depth sufficiently small for repeated hardware execution. For each experiment, a fixed total budget corresponding to $40$ shots per independent off-diagonal kernel entry was used. 
Adaptive allocation was initialized with a pilot allocation of $8$ shots per entry and followed by three adaptive allocation rounds using the multinomial allocation rule with $\lambda=0.5$.

The complete adaptive-versus-uniform comparison was repeated over 15 independent hardware runs. Each run corresponds to a fully independent execution of the adaptive measurement process on the quantum device, with fresh kernel measurements collected from hardware at every stage of the allocation procedure. In particular, the reported results do not rely on classical resampling, bootstrapping, or post-hoc simulation of measurement outcomes from a fixed dataset. Instead, each run constitutes a genuine end-to-end realization of the adaptive learning pipeline under independently observed hardware noise and finite-shot sampling fluctuations. Consequently, the observed performance differences reflect the behavior of the allocation strategy under realistic quantum hardware conditions rather than synthetic reconstructions of the measurement process.
The small problem size was chosen deliberately. The goal of the hardware validation was to isolate the effect of measurement allocation under realistic quantum noise while allowing repeated executions under a fixed resource budget. The selected dataset is sufficiently structured to produce a non-trivial support-vector configuration while remaining compatible with repeated hardware evaluation.

To minimize the impact of device calibration drift and temporal fluctuations, adaptive and uniform experiments were executed in a paired fashion. For each hardware run, the adaptive and uniform allocations were evaluated sequentially within the same execution session and under the same backend calibration state whenever possible. Consequently, the reported comparisons primarily reflect differences in measurement allocation strategy rather than long-term variations in hardware performance.

Figure~\ref{fig:appendix_ibm_allocations} shows the average measurement allocations observed across the IBM hardware experiments. While the uniform baseline distributes shots evenly across all kernel entries, the adaptive strategy concentrates measurements on a small subset of entries and reduces allocation elsewhere. This behavior is consistent with the objective of prioritizing kernel entries that have the greatest influence on the resulting SVM classifier.

Table~\ref{tab:appendix_ibm_run_by_run} reports the individual outcomes of the 15 paired IBM Quantum hardware runs used in the hardware validation study. 
For each independent run, the table lists the final performance obtained with the adaptive allocation strategy and the corresponding uniform-allocation baseline executed under the same hardware conditions. 

The run-level results provide a more detailed view of the aggregate statistics reported in Sec.~\ref{sec:ibm_hardware_validation}. 
While uniform allocation consistently achieves lower full-kernel reconstruction error, adaptive allocation generally improves the quantities most relevant to the downstream SVM classifier, including support-vector reconstruction, weighted support-vector agreement, and decision-function estimation. 
Importantly, these improvements are not driven by a small number of exceptional runs but are observed repeatedly across the hardware executions. 

The table also illustrates the natural run-to-run variability arising from finite-shot sampling and quantum hardware noise. 
Nevertheless, the qualitative adaptive-versus-uniform trade-off remains remarkably stable: uniform allocation tends to provide the most accurate approximation of the complete Gram matrix, whereas adaptive allocation preferentially improves the kernel entries that determine the learned classifier.

\begin{table*}[t]
\centering
\caption{
Run-by-run IBM hardware results for the final adaptive allocation and the uniform baseline. Each entry reports adaptive/uniform performance for the same paired hardware run. Lower values are better for RMSE and margin-error metrics, whereas higher values are better for weighted Jaccard.
}
\label{tab:appendix_ibm_run_by_run}
\resizebox{\textwidth}{!}{%
\begin{tabular}{lccccc}
\toprule
Run & $\mathrm{RMSE}(K)$ & $\mathrm{RMSE}(K_{\mathrm{SV}})$ & $J_{\mathrm{SV}}^{w}$ & Margin err. & Decision RMSE \\
\midrule
run1 & $0.0605/0.0584$ & $0.0111/0.0338$ & $0.9209/0.7583$ & $0.2732/1.0136$ & $0.1611/0.5976$ \\
run2 & $0.1140/0.0741$ & $0.0064/0.0054$ & $0.9561/0.9133$ & $0.1424/0.3194$ & $0.0799/0.2508$ \\
run3 & $0.0595/0.0619$ & $0.0218/0.0270$ & $0.8408/0.7660$ & $0.5572/1.0693$ & $0.3358/0.6443$ \\
run4 & $0.0783/0.0521$ & $0.0082/0.0163$ & $0.9294/0.8679$ & $0.2563/0.4794$ & $0.1582/0.2959$ \\
run5 & $0.0564/0.0484$ & $0.0206/0.0252$ & $0.8493/0.8216$ & $0.5259/0.6225$ & $0.3220/0.3812$ \\
run6 & $0.0905/0.0552$ & $0.0060/0.0115$ & $0.9559/0.9185$ & $0.1456/0.2690$ & $0.0889/0.1643$ \\
run7 & $0.0850/0.0434$ & $0.0006/0.0498$ & $0.9951/0.7185$ & $0.0163/0.9393$ & $0.0097/0.5581$ \\
run8 & $0.0767/0.0497$ & $0.0086/0.0782$ & $0.9398/0.5940$ & $0.1956/1.0851$ & $0.1157/0.6386$ \\
run9 & $0.0607/0.0525$ & $0.0180/0.0157$ & $0.8033/0.8710$ & $0.1657/0.5046$ & $0.1873/0.3112$ \\
run10 & $0.0883/0.0461$ & $0.0082/0.0857$ & $0.9392/0.5799$ & $0.2033/1.3004$ & $0.1215/0.7773$ \\
run11 & $0.0941/0.0685$ & $0.0115/0.0168$ & $0.9142/0.8635$ & $0.2913/0.5362$ & $0.1728/0.3182$ \\
run12 & $0.0640/0.0491$ & $0.0056/0.0357$ & $0.9570/0.7123$ & $0.1452/1.3633$ & $0.0884/0.8303$ \\
run13 & $0.0700/0.0514$ & $0.0084/0.0254$ & $0.9325/0.8207$ & $0.2355/0.6258$ & $0.1452/0.3858$ \\
run14 & $0.0868/0.0633$ & $0.0117/0.0342$ & $0.9166/0.7562$ & $0.2890/1.0238$ & $0.1621/0.5743$ \\
run15 & $0.0715/0.0600$ & $0.0211/0.0104$ & $0.8433/0.9163$ & $0.5535/0.2956$ & $0.3402/0.1817$ \\
\bottomrule
\end{tabular}%
}
\end{table*}
\end{appendices}
\end{document}